\documentclass[lettersize,journal]{IEEEtran}
\usepackage{cite}
\usepackage{amsmath,amssymb,amsfonts}
\usepackage{amsthm}
\usepackage{graphicx}
\usepackage{textcomp}
\usepackage{array}
\usepackage{textcomp}
\usepackage{stfloats}
\usepackage{url}
\usepackage{verbatim}
\usepackage{graphicx}
\usepackage{cite}
\usepackage{flushend}
\usepackage[linesnumbered,ruled]{algorithm2e}

\SetCommentSty{smallcommentsty}
\usepackage{algorithmicx}
\usepackage{bbm}
\usepackage{subcaption}
\usepackage{xcolor}
\usepackage[colorlinks=true, linkcolor=blue, citecolor=blue, urlcolor=blue]{hyperref}

\newtheorem{theorem}{Theorem}
\newtheorem{proposition}{Proposition}

\newtheorem{lemma}{Lemma}
\theoremstyle{definition}
\newtheorem{definition}{Definition}
\theoremstyle{remark}
\newtheorem{remark}{Remark}

\def\BibTeX{{\rm B\kern-.05em{\sc i\kern-.025em b}\kern-.08em
    T\kern-.1667em\lower.7ex\hbox{E}\kern-.125emX}}
\begin{document}

\title{Safe Multi-Agent Reinforcement Learning with Convergence to Generalized Nash Equilibrium}

\author{Zeyang Li and Navid Azizan
    \thanks{Zeyang Li and Navid Azizan are with the Laboratory for Information and Decision Systems, Massachusetts Institute of Technology, Cambridge, MA 02139, United States. (emails: zeyang@mit.edu; azizan@mit.edu).}
    \thanks{Code is available at \url{https://github.com/azizanlab/SafeMARL}.}
}

\maketitle

\begin{abstract}
    Multi-agent reinforcement learning (MARL) has achieved notable success in cooperative tasks, demonstrating impressive performance and scalability. However, deploying MARL agents in real-world applications presents critical safety challenges. Current safe MARL algorithms are largely based on the constrained Markov decision process (CMDP) framework, which enforces constraints only on discounted cumulative costs and lacks an all-time safety assurance. Moreover, these methods often overlook the feasibility issue—where the system will inevitably violate state constraints within certain regions of the constraint set—resulting in either suboptimal performance or increased constraint violations.
    To address these challenges, we propose a novel theoretical framework for safe MARL with \emph{state-wise} constraints, where safety requirements are enforced at every state the agents visit. To resolve the feasibility issue, we leverage a control-theoretic notion of the feasible region, the controlled invariant set (CIS), characterized by the safety value function. We develop a multi-agent method for identifying CISs, ensuring convergence to a Nash equilibrium (i.e., ``person-by-person optimality'') on the safety value function. By incorporating CIS identification into the learning process, we introduce a multi-agent dual policy iteration algorithm that guarantees convergence to a generalized Nash equilibrium (i.e., person-by-person optimality subject to joint safety requirements) in state-wise-constrained cooperative Markov games, achieving an optimal balance between feasibility and performance.
    Furthermore, for practical deployment in complex high-dimensional systems, we propose \emph{Multi-Agent Dual Actor-Critic} (MADAC), a safe MARL algorithm that approximates the proposed iteration scheme within the deep RL paradigm. Empirical evaluations on safe MARL benchmarks demonstrate that MADAC consistently outperforms existing methods, delivering much higher rewards while reducing constraint violations.
\end{abstract}


\section{Introduction}
Cooperative multi-agent reinforcement learning (MARL) has drawn substantial research attention, driven by the increasing demand for effective coordination among intelligent machines \cite{zhang2021multi, shengbo2018reinforcement}. Agents in a cooperative MARL system collaborate with each other to pursue a joint objective, operating in a shared environment where each agent's decisions impact the entire team. To tackle the non-stationarity challenge of one agent's updates affecting others, the centralized training decentralized execution (CTDE) paradigm has been developed \cite{lowe2017multi, zhou2023malib}. In the CTDE framework, the algorithm has access to global information across all agents during the training phase, after which the trained agents are prepared to make independent decisions for deployment. Significant advancements have been made in this area, both in theory \cite{rashid2020monotonic, bertsekas2021multiagent, bhattacharya2023multiagent, zhong2024heterogeneous, liu2024maximum} and applications \cite{vinyals2019grandmaster, chu2019multi, yu2022surprising, novati2021automating}.

Despite the remarkable achievements in MARL algorithms, most of them are primarily designed to maximize rewards and achieve optimal performance. However, real-world deployments demand that agents operate within specific constraints, i.e., ensuring safety. In practice, safety enforcement should take precedence over performance optimization. For example, in the context of autonomous driving, vehicles must follow traffic rules and avoid collisions under all circumstances, prioritizing safety above all else. While safe RL has been extensively studied in the single-agent setting \cite{zhao2023state}, the safe MARL problem remains quite challenging. That is largely due to the fact that in the multi-agent setting, both the objective and the constraints are influenced by the decisions of all agents. There are multiple points of tension that could destabilize the learning process, including the inherent conflict between reward maximization and constraint satisfaction, as well as disagreements in updating directions across agents that could violate the shared constraints.

Pioneering works in safe MARL theory include \cite{melcer2022shield, lu2021decentralized, gu2023safe, ding2023provably, ying2024scalable, zhao2024multi}. Despite the significant progress, there are notable limitations with existing safe MARL methods. First, many algorithms, such as Safe Dec-PG \cite{lu2021decentralized}, MACPO and MAPPO-Lagrangian \cite{gu2023safe} are grounded in the constrained Markov decision process (CMDP) framework \cite{altman1999constrained}, in which the objective is to maximize total rewards while ensuring that the expected trajectory costs remain below a predefined threshold. However, real-world applications often require \emph{state-wise} constraints, i.e., safety constraints that are enforced at every time step rather than just on average, as any violation could lead to catastrophic consequences \cite{zhao2023state}. Second, existing methods rarely consider the feasibility issue, i.e., wherein no policy can keep the system persistently safe for certain states within a subset of the constraint set \cite{margellos2011hamilton}. Identifying the feasible region is critical for safe RL algorithms, as optimizing rewards over infeasible regions is ill-posed and can lead to unstable training \cite{yu2022reachability, ganai2023iterative, li2024safe}. Additionally, a larger feasible region provides a broader operational space for the agents, leading to better performance. Third, it is particularly challenging to coordinate the updates of multiple agents to ensure joint constraint satisfaction while still improving overall performance. Since each agent’s update alters the set of feasible directions for the others, the update mechanism must be carefully designed to ensure that agents collaboratively converge to a safe and performant equilibrium. Existing methods often struggle to balance constraint satisfaction with task performance, resulting in either suboptimal outcomes or increased violations of safety constraints.

To address the aforementioned challenges, this paper introduces a novel theoretical framework for handling feasibility and optimality in safe MARL with state-wise constraints. Unlike the CMDP setting commonly used in existing studies, which imposes trajectory-wise constraints, we focus on a stricter safety formulation where agents must satisfy constraints at every state they visit. For identification of feasible region, we utilize the concept of controlled invariant set in control theory, characterized by the safety value function from Hamilton-Jacobi reachability analysis \cite{margellos2011hamilton, bansal2017hamilton}. Inspired by the one-agent-at-a-time policy improvement technique \cite{bertsekas2021multiagent, bhattacharya2023multiagent, zhong2024heterogeneous, liu2024maximum}, we propose a MARL approach for the identification of controlled invariant sets, which ensures convergence to a Nash equilibrium on safety value function, i.e., the joint safety policy is ``person-by-person optimal'': no agent can improve the safety value function by unilaterally adjusting its own policy. By leveraging action-space restrictions derived from these controlled invariant sets, we develop a multi-agent dual policy iteration algorithm for constrained cooperative Markov games. The objective is two-fold: for states within the identified controlled invariant set, agents aim to maximize the joint rewards, whereas, for states outside this set, the focus shifts to minimizing constraint violations. The proposed iteration scheme employs two individual policies for each agent—one dedicated to achieving the two-fold objective and the other for collaboratively learning the controlled invariant sets. Our methods enable agents to obtain the highest rewards possible while ensuring safety, in the sense of a generalized Nash equilibrium, i.e., the joint policy is person-by-person optimal: no agent can improve the two-fold objective by unilaterally adjusting its own policy without violating the shared safety constraints. The key contributions of this paper are summarized as follows.
\begin{itemize}
    \item We propose a MARL approach for identifying controlled invariant sets of nonlinear dynamical systems. Our method ensures convergence to a Nash equilibrium on the safety value function, achieving local optimality in finding feasible regions within the multi-agent setting.
    \item We introduce the multi-agent dual policy iteration algorithm, which, to the best of our knowledge, is the first safe MARL algorithm with guaranteed convergence to a generalized Nash equilibrium on state-wise-constrained cooperative Markov games. The key insight is to transform state constraints into state-dependent action spaces, which allows optimizing task performance across all feasible directions in an agent-by-agent manner, despite the inherent non-stationarity in safe MARL problems.
    \item  We propose multi-agent dual actor-critic (MADAC), a safe MARL algorithm that approximates the proposed iteration scheme within the deep RL context. Empirical evaluations on safe MARL benchmarks demonstrate that MADAC consistently outperforms existing methods, delivering much higher rewards while reducing constraint violations.
\end{itemize}

\section{Related Work}
\textbf{Safe RL} has been extensively studied in the single-agent setting. There are mainly two types of safety formulation in safe RL algorithms: trajectory-wise safety and state-wise safety. The former requires that the expected trajectory cost remains below a certain threshold (the CMDP framework \cite{altman1999constrained}), while the latter enforces constraints at every state the agents visit \cite{zhao2023state}.
For trajectory-wise safety formulation, algorithms typically employ the Lagrange method or trust region method to perform constrained policy optimization. Ha et al. \cite{ha2020learning} enhance the soft actor-critic algorithm by adding Lagrange multipliers and performing dual ascent on the Lagrangian. Chow et al. \cite{chow2017risk} introduce constraints on the conditional value-at-risk of cumulative costs and design corresponding policy gradient and actor-critic algorithms. Achiam et al. \cite{achiam2017constrained} propose the constrained policy optimization algorithm with guarantees for near constraint satisfaction at each iteration. Zhang et al. \cite{zhang2020first} update policy in the nonparameterized policy space and then project the updated policy back into the parametric policy space.
For state-wise safety formulation, the key is to achieve set invariance with control-theoretic tools such as Hamilton-Jacobi reachability \cite{margellos2011hamilton, hsu2023safety}, control barrier function \cite{ames2019control}, and safety index \cite{liu2014control}. Ma et al. \cite{ma2022joint} jointly optimize the control policy and the neural safety index. Yu et al. \cite{yu2022reachability} jointly learn the safety value function and the control policy. Ganai et al. \cite{ganai2023iterative} address state-wise constraints for stochastic dynamics. Li et al. \cite{li2024safe} propose a dual policy iteration algorithm that converges to the maximal robust invariant set and optimal task policy simultaneously. Wang et al. \cite{wang2024magics} propose a robust safe RL algorithm that guarantees local convergence to a minimax equilibrium solution. State-wise safety has also been explored from optimization perspectives. Zhao et al. \cite{zhao2024statewise} propose state-wise constrained policy optimization, demonstrating that the algorithm achieves bounded worst-case safety violations in state-wise constrained MDPs.

\textbf{Safe MARL} is an emerging research area that has been gaining increasing attention.
Lu et al. \cite{lu2021decentralized} propose a decentralized policy gradient approach for constrained policy optimization with networked agents and prove the algorithm converges to a first-order stationarity point for the formulated problem.
Melcer et al. \cite{melcer2022shield} address the shielded reinforcement learning problem in the multi-agent setting and present an algorithm for the decomposition of a centralized shield.
Gu et al. introduce trust region optimization techniques that ensure monotonic improvement in rewards while satisfying safety constraints at every iteration. They also present two algorithms (MACPO and MAPPO-Lagrangian) for practical deployment in high-dimensional environments \cite{gu2023safe}.
Zhao et al. \cite{zhao2024multi} update policies by solving a constrained optimization problem in the non-parameterized policy space and then projecting the solution into the parametric policy space.
Ding et al. \cite{ding2023provably} propose a provably efficient algorithm for solving episodic two-player zero-sum constrained Markov games.
Ying et al. \cite{ying2024scalable} present a primal-dual approach for safe MARL problems, where the objective and constraints are defined by general utilities, and prove its convergence to a first-order stationary point.
Our work differs from existing methods in several key aspects. First, we consider a stricter safety requirement that enforces constraints at every state the agents visit, contrasting with the trajectory-wise safety formulation (the CMDP framework) employed in \cite{lu2021decentralized, gu2023safe, zhao2024multi, ding2023provably, ying2024scalable}.
Second, we address the feasibility issue inherent in safe RL problems \cite{yu2022reachability, li2024safe}. Our algorithm automatically identifies the feasible region and adjusts to different objectives based on the state's feasibility, whereas existing methods \cite{lu2021decentralized, gu2023safe, zhao2024multi, ding2023provably, ying2024scalable} generally assume initial feasibility. Third, although the algorithm in \cite{gu2023safe} achieves performance improvement and constraint satisfaction at each iteration, no convergence guarantee is provided. Our method ensures convergence to a generalized Nash equilibrium, where no agent can enhance the objective by unilaterally modifying its policy within all feasible directions.

\section{Multi-Agent Dual Policy Iteration}

In this section, we present our problem formulation and algorithm design for cooperative MARL with state-wise constraints, along with the corresponding convergence guarantees. These results constitute the theoretical core of this paper. We begin with the problem setting and the main challenges it poses, which motivate the developments that follow.

Throughout this section, we consider finite state and action spaces, and all theoretical results are established in this setting. We have the following definition.

\begin{definition}[state-wise-constrained cooperative Markov game]
    A state-wise-constrained cooperative Markov game is represented by a tuple $\mathcal{M} = \left(\mathcal{N}, \mathcal{X}, \mathcal{U}, f, r, h, \gamma, \gamma_h, d\right)$, where $\mathcal{N} = \{1, 2, \dots, n\}$ denotes the set of $n$ agents, $\mathcal{X}$ denotes the state space, $\mathcal{U} = \mathcal{U}_1\times\mathcal{U}_2\times\cdots\times\mathcal{U}_n$ is the joint action space (as the product of individual action spaces), $f: \mathcal{X} \times \mathcal{U} \rightarrow \mathcal{X}$ denotes the deterministic system dynamics, $r: \mathcal{X} \times \mathcal{U} \rightarrow \mathbb{R}$ is the reward function, $h: \mathcal{X} \rightarrow \mathbb{R}$ is the constraint function, $\gamma \in (0, 1)$ denotes the reward discount factor, $\gamma_h \in (0, 1)$ denotes the safety discount factor, and $d$ denotes the initial state distribution.
\end{definition}

Note that the safety discount factor $\gamma_h$ is introduced specifically for algorithms addressing safety, which will be explained later. If there are multiple safety requirements, they can be merged into one constraint function as $h(x) = \min \left\{h_1(x), h_2(x), \cdots\right\}$.
Given a state $x\in \mathcal{X}$, each agent individually selects an action $u_i \in \mathcal{U}_i$. An individual policy is denoted by $\pi_i: \mathcal{X} \rightarrow \mathcal{U}_i$, and $u_i=\pi_i(x)$. The joint action of all agents is denoted by $u = (u_1, u_2, \cdots, u_n)\in \mathcal{U}$. The joint policy is expressed as $\pi=(\pi_1, \pi_2, \cdots, \pi_n)$, and thus $u=\pi(x)$. The set of all individual policies of agent $i$ is denoted by $\Pi_i=\{\pi_i:\mathcal{X}\rightarrow\mathcal{U}_i\}$, and the set of all joint policies by $\Pi=\Pi_1\times\Pi_2\times\cdots\times\Pi_n$.

Since our method involves operations on the actions and policies of a subset of agents, the following notation is introduced. An ordered subset of $\mathcal{N}$, $\left\{i_1, i_2,\cdots,i_m\right\}$, is denoted as $i_{1:m}$. The joint policy of agents $i_{1:m}$ is denoted by $\pi_{i_{1:m}}$, and their joint action by $u_{i_{1:m}}$. For an agent $i$, the subscript $-i$ refers to all the other agents, so that $\pi=(\pi_i,\pi_{-i})$ and $u=(u_i,u_{-i})$.

In a state-wise-constrained cooperative Markov game, the agents aim to collectively maximize the discounted cumulative reward while satisfying the constraint $h(x_t) \geq 0$ at every time step. The ideal formulation is therefore
\begin{equation*}
    \begin{gathered}
        \max _{\pi\in\Pi} \mathop{\mathbb{E}}\limits_{x_0\sim d}\left\{\sum_{t=0}^{\infty}\gamma^t r\left(x_t, u_t\right)\right\} \\
        \begin{aligned}
            \text { s.t. \quad }&x_{t+1}=f\left(x_t, u_t\right), u_t= \pi\left( x_t \right),\\& h\left(x_t\right) \geq 0, \ \forall t\geq0.
        \end{aligned}
    \end{gathered}
\end{equation*}
This problem poses three major challenges. The first is \emph{feasibility}: for certain states within the constraint set, no joint policy can keep the system persistently safe, i.e.,
\begin{equation*}
    \exists\, x_0 \ \text{with} \ h(x_0)\geq 0 \ \ \text{s.t.} \ \ \forall \pi\in\Pi, \ \exists\, T\in\mathbb{N}: \ h(x_T)<0,
\end{equation*}
so optimizing rewards from these states is ill-posed. Hence, a proper formulation must be established before an algorithm can be designed.
The second is \emph{scalability}: the joint action space grows exponentially with the number of agents, so a policy update that optimizes over joint actions quickly becomes computationally intractable, and the agents must instead update their policies individually. The third is \emph{coordination}: under the shared constraint, each agent's update alters the feasible actions of the others, so these updates must be carefully coordinated to remain jointly feasible while improving the shared objective.
To address these challenges, we first develop a multi-agent approach to identify controlled invariant sets for nonlinear systems, which plays a critical role in the following formulation and analysis. Next, we provide a formal problem formulation for state-wise-constrained cooperative Markov games, specifying the objective and constraints. Finally, we present the iteration scheme and establish its convergence to a generalized Nash equilibrium for the constructed problem.

\subsection{Identification of Controlled Invariant Set}

As discussed earlier, a key fact from control theory is that, given a constraint set, only a subset of it can be controlled to remain safe indefinitely, which is known as the controlled invariant set (CIS). Therefore, to design safe RL algorithms, it is essential to properly identify the CIS and to constrain the system state within the CIS rather than the full constraint set.
Otherwise, when the system state leaves the CIS, it can violate the constraints regardless of the control inputs, thus causing the constrained policy optimization to be ill-conditioned.

Hamilton-Jacobi reachability analysis is a control-theoretic verification method for ensuring the safety of nonlinear systems. In this approach, the CIS is identified through the safety value function, which is the solution to the Hamilton-Jacobi-Isaacs partial differential equation (HJI-PDE) \cite{margellos2011hamilton, bansal2017hamilton}.
However, solving the PDEs requires discretizing the state and action spaces, which leads to exponential scaling with increasing dimensions (the curse of dimensionality).
Recently, RL approaches have been developed for Hamilton-Jacobi reachability, providing more tractable solutions for high-dimensional systems \cite{fisac2019bridging, yu2022reachability, li2024safe}. However, these methods are limited to single-agent systems, and it remains unclear how to extend them to multi-agent settings as well as providing convergence guarantees. We will address these gaps by proposing an agent-by-agent sequential update scheme for safety value functions.
Given some safety requirements $h(x)\geq0$, we define the safety value function and the controlled invariant set as follows.

\begin{definition}[constraint set]
    The constraint set is defined as the zero-superlevel set of constraint function $h(x)$:
    \begin{equation*}
        S_h=\left\{x\in \mathcal{X} \mathrel{}\middle|\mathrel{} h(x)\geq 0\right\}.
    \end{equation*}
\end{definition}

\begin{definition}[safety value function \cite{yu2022reachability, li2024safe}]
    \label{definition of safety value}
    Given a joint policy $\pi_h$ of all agents and $\gamma_h\in (0,1)$, the safety value function $V_h^{\pi_h}$ of this specific policy is defined as
    \begin{equation*}
        \begin{gathered}
            V_h^{\pi_h}(x)=\inf _{t \in \mathbb{N}}\left\{\gamma_h^{t+1}h\left(x_t\right)\right\} \\
            \begin{aligned}
                \text { s.t. \quad }&x_0=x,u_t=\pi_h\left(x_t\right),\\
                &x_{t+1}=f\left(x_{t}, u_{t}\right),t\geq0.
            \end{aligned}
        \end{gathered}
    \end{equation*}
\end{definition}

\begin{definition}[controlled invariant set \cite{li2024safe}]
    \label{Controlled Invariant Set}
    If $\mathop{\max}\limits_{x\in \mathcal{X}} V_h^{\pi_h}(x)\geq 0$ for a joint policy $\pi_h$, the controlled invariant set of this specific policy is defined as
        \begin{equation*}
            S_{\rm{c}}^{\pi_{h}}=\left\{x \in \mathcal{X} \mathrel{}\middle|\mathrel{} V_h^{\pi_h}(x)\geq 0\right\}.
        \end{equation*}
\end{definition}

The safety value function $V_h^{\pi_h}$ quantifies the system's risk under a given policy $\pi_h$ by capturing whether any constraints are violated along the infinite-horizon trajectory induced by $\pi_h$.
The discount factor $\gamma_h$ is introduced to ensure that the safety value is the fixed point of a contraction mapping (which is defined in Lemma~\ref{lemma for safety value function}).
Starting from any state with a nonnegative safety value, the system is guaranteed to remain safe under policy $\pi_h$, since $\inf _{t \in \mathbb{N}}\left\{\gamma_h^{t+1}h\left(x_t\right)\right\}\ge 0$ is equivalent to requiring $h\left(x_t\right)\geq 0$ for all $t\in \mathbb{N}$.
This leads to the concept of controlled invariance, as stated in Definition~\ref{Controlled Invariant Set}, which we will further explore later.

Since $V_h^{\pi_h}$ is defined over the infinite horizon, it inherently follows a recursive structure arising from dynamic programming, referred to as the self-consistency condition \cite{fisac2019bridging, yu2022reachability, li2024safe}.

\begin{lemma}[self-consistency condition for safety value function]
    \label{lemma for safety value function}
    The safety value function $V_h^{\pi_h}$ satisfies
    \begin{equation}
        \label{Self-consistency Condition for safety value function}
        V_h^{\pi_h}(x)=\gamma_h \min \left\{h(x), V_h^{\pi_h}\left(f(x,\pi_h(x))\right)\right\}.
    \end{equation}
    Furthermore, the safety self-consistency operator $\mathcal{T}_h^{\pi_h}$ defined as
    \begin{equation*}
        \left[\mathcal{T}_h^{\pi_h} (V_h)\right](x)=\gamma_h \min \left\{h(x), V_h\left(f(x,\pi_h(x))\right)\right\}
    \end{equation*}
    is a contraction mapping and $V_h^{\pi_h}$ is the unique fixed point of $\mathcal{T}_h^{\pi_h}$.
\end{lemma}

\begin{proof}
Fix a joint safety policy $\pi_h$ and an initial state $x_0=x$. Let $x_{t+1}=f(x_t,\pi_h(x_t))$ and denote $x_1=f(x,\pi_h(x))$.
By Definition~\ref{definition of safety value},
\begin{equation*}
    V_h^{\pi_h}(x)
    =\inf_{t\ge 0}\gamma_h^{t+1}h(x_t)
    =\min\Big\{\gamma_h h(x_0),\ \inf_{t\ge 1}\gamma_h^{t+1}h(x_t)\Big\}.
\end{equation*}
For the second term, re-index with $t=s+1$:
\begin{equation*}
    \begin{aligned}
        \inf_{t\ge 1}\gamma_h^{t+1}h(x_t)&
        =\inf_{s\ge 0}\gamma_h^{s+2}h(x_{s+1})
        =\gamma_h \inf_{s\ge 0}\gamma_h^{s+1}h(x_{s+1}) \\&
        =\gamma_h V_h^{\pi_h}(x_1).
    \end{aligned}
\end{equation*}
Therefore,
\begin{equation*}
    \begin{aligned}
        V_h^{\pi_h}(x) &
        =\min\{\gamma_h h(x),\ \gamma_h V_h^{\pi_h}(f(x,\pi_h(x)))\} \\ &
        =\gamma_h \min\{h(x),V_h^{\pi_h}(f(x,\pi_h(x)))\},
    \end{aligned}
\end{equation*}
which proves the self-consistency condition.

Let $V,W$ be two bounded functions on $\mathcal{X}$. Using that $|\min\{a,b\}-\min\{a,c\}|\le |b-c|$ and that $f(x,\pi_h(x))\in\mathcal{X}$, we have
\begin{equation*}
    \begin{aligned}
        & \big|\mathcal{T}_h^{\pi_h}(V)(x)-\mathcal{T}_h^{\pi_h}(W)(x)\big| \\&
        \le \gamma_h \big|V(f(x,\pi_h(x)))-W(f(x,\pi_h(x)))\big|
        \le \gamma_h \|V-W\|_\infty.
    \end{aligned}
\end{equation*}
Taking supremum over $x$ yields $\|\mathcal{T}_h^{\pi_h}(V)-\mathcal{T}_h^{\pi_h}(W)\|_\infty\le \gamma_h\|V-W\|_\infty$, so $\mathcal{T}_h^{\pi_h}$ is a contraction. Uniqueness of the fixed point follows from the contraction mapping theorem.
\end{proof}

\begin{remark}
    Note that there are multiple ways to introduce the discount factor $\gamma_h$ in the safety value function, such as \cite{fisac2019bridging, li2022infinite, so2024train}, leading to slightly different formulations of the self-consistency condition. For simplicity of presentation, we choose to omit the $(1-\gamma_h)h(x)$ term commonly used in the literature \cite{fisac2019bridging, yu2022reachability, li2024safe} since the resulting form is still a contraction. Our proposed method readily extends to these alternative formulations as well.
\end{remark}

Solving the operator equation $\mathcal{T}_h^{\pi_h} (V_h) = V_h$ is essentially the policy evaluation step, in which we calculate the safety value function $V_h^{\pi_h}$ for a given policy $\pi_h$. Naturally, the next step is policy improvement. In the single-agent setting, this can be done as
\begin{equation}
    \label{Single-agent Safety Policy Improvement}
    \pi_h^{\rm{new}}(x)=\underset{u \in \mathcal{U}}{\operatorname{argmax}}\left\{\gamma_h \min \left\{h(x), V_h^{\pi_h}\left(f(x,u)\right)\right\}\right\}.
\end{equation}
However, this maximization step can become intractable as the number of agents increases. Suppose there are $n$ agents and each of them has $C$ admissible actions, then the computational complexity of (\ref{Single-agent Safety Policy Improvement}) is $O(C^n)$, which scales exponentially with the number of agents. To address this issue, we propose an agent-by-agent sequential update scheme for optimizing the safety value function with a computational complexity of $O(Cn)$, while still preserving the convergence guarantee. The proposed multi-agent policy iteration approach for identification of CIS is summarized in Algorithm~\ref{Multi-agent safety policy iteration}.
The term \textit{safety} in the names emphasizes that these policies are designed to be as safe as possible, without regard to reward maximization, distinguishing them from the \textit{task} policies introduced later.
Note that in each $\operatorname{argmax}$ operation we assume a deterministic tie-breaking rule (e.g., retaining the current action whenever it is among the maximizers) to prevent policy updates from oscillating without yielding any actual improvement in the values.

\begin{algorithm*}[htbp]
    \caption{Multi-agent safety policy iteration for identification of controlled invariant set}
    \label{Multi-agent safety policy iteration}
    \KwIn{Initial joint safety policy $\pi_h$.}
    \For{$k$ iterations}{
        \tcp{Joint safety policy evaluation}
        Solve $V_h^{\pi_h} = \mathcal{T}_h^{\pi_h}\left(V_h^{\pi_h}\right)$ for $V_h^{\pi_h}$.\label{Safety policy evaluation}

        \tcp{Multi-agent safety policy improvement}
        Randomly shuffle the order of agents $\mathcal{N}$ as $i_{1:n}$.\label{first line of full safety policy improvement}

        Perform the following agent-by-agent sequential update:

        \For{each $x \in \mathcal{X}$}{
            $\pi_{h,i_1}^{\rm{new}}(x) = \underset{u_{i_1} \in \mathcal{U}_{i_1}}{\operatorname{argmax}} \left\{ \gamma_h \min \left\{h(x), V_h^{\pi_h}\left(f\left(x, \left(u_{i_1}, \pi_{h,i_{2:n}}(x)\right)\right)\right)\right\}\right\}$\label{first line of safety policy improvement}

            $\pi_{h,i_2}^{\rm{new}}(x) = \underset{u_{i_2} \in \mathcal{U}_{i_2}}{\operatorname{argmax}} \left\{ \gamma_h \min \left\{h(x), V_h^{\pi_h}\left(f\left(x, \left(\pi_{h,i_1}^{\rm{new}}(x), u_{i_2}, \pi_{h,i_{3:n}}(x)\right)\right)\right)\right\}\right\}$

            $\ \quad\vdots$

            $\pi_{h,i_n}^{\rm{new}}(x) = \underset{u_{i_n} \in \mathcal{U}_{i_n}}{\operatorname{argmax}} \left\{ \gamma_h \min \left\{h(x), V_h^{\pi_h}\left(f\left(x, \left(\pi_{h,i_{1:n-1}}^{\rm{new}}(x), u_{i_n}\right)\right)\right)\right\}\right\}$\label{last line of safety policy improvement}
        }\label{last line of full safety policy improvement}

        Update the joint safety policy as $\pi_h\leftarrow \pi_h^{\rm{new}}$.
    }
\end{algorithm*}

The core of the proposed multi-agent safety policy improvement is to sequentially optimize the individual safety policy of each agent, utilizing the most up-to-date policies available. This results in a coordinate-descent-style updating scheme.
Our method is inspired by the one-agent-at-a-time technique used in previous MARL algorithms \cite{bertsekas2021multiagent, bhattacharya2023multiagent, zhong2024heterogeneous, liu2024maximum}, which focuses on optimizing the standard value function (cumulative rewards). In this paper, we show that this approach also works for the safety problem, despite its structural difference from the reward case.

Next, we establish the convergence of the multi-agent safety policy iteration and characterize the optimality of the joint safety policy to which it converges. In cooperative settings, where all agents share a common objective, there are two standard notions of optimality \cite{radner1962team, yuksel2017convex, charalambous2017centralized, marden2018game}. A joint policy is \emph{team-optimal} if it attains the global optimum of the common objective over all joint policies. It is \emph{person-by-person optimal} if no agent can improve the common objective by unilaterally changing its own policy, which is exactly the Nash equilibrium condition when the common objective is regarded as the payoff of every agent \cite{yuksel2017convex, sanjari2023isomorphism}. In this paper, we focus on person-by-person optimality, since team optimality typically requires stronger structural assumptions, such as convexity \cite{radner1962team, krainak1982static, yuksel2017convex}, or policy updates over the joint action space \cite{bertsekas2021multiagent}. We state the definition of the Nash equilibrium below. Since the problem studied in this paper also involves constraints shared by the agents, the definition is given together with its extension to the constrained case, the generalized Nash equilibrium, which will be used in Section~\ref{Iteration Scheme}.

\begin{definition}[Nash equilibrium and generalized Nash equilibrium \cite{facchinei2010generalized}]
    \label{Equilibria}
    Let $J:\Pi\to\mathbb{R}$ denote the shared objective to be maximized by the agents, and let $\mathcal{F}\subseteq\Pi$ be a nonempty set of feasible joint policies. A joint policy $\pi^*\in\mathcal{F}$ is a generalized Nash equilibrium if, for every agent $i\in\mathcal{N}$ and every policy $\pi_i\in\Pi_i$ satisfying $(\pi_i,\pi_{-i}^*)\in\mathcal{F}$,
    \begin{equation*}
        J(\pi_i,\pi_{-i}^*)\leq J(\pi^*).
    \end{equation*}
    Thus, no agent can improve the common objective through a feasible unilateral policy change. When $\mathcal{F}=\Pi$, this condition reduces to that of a Nash equilibrium.
\end{definition}

The safety policies are unconstrained, so we have $\mathcal{F}=\Pi$ for Algorithm~\ref{Multi-agent safety policy iteration} and the relevant notion is the Nash equilibrium. The following theorem shows that Algorithm~\ref{Multi-agent safety policy iteration} converges to such an equilibrium, i.e., to a person-by-person optimal joint safety policy.

\begin{theorem}[convergence of multi-agent safety policy iteration]
    \label{Convergence of Multi-agent Safety Policy Iteration}
    Algorithm~\ref{Multi-agent safety policy iteration} converges to a set of individual safety policies $\left\{\pi_{h,1}^*,\pi_{h,2}^*, \cdots,\pi_{h,n}^*\right\}$ that collectively achieve a Nash equilibrium, i.e., no agent can unilaterally modify its individual safety policy to improve the safety value function of the joint safety policy.
\end{theorem}

\begin{proof}
    Suppose $i_{1:n}$ is the optimization order of agents in the $k$-th iteration. $\forall x\in \mathcal{X}$,  we have
    \begin{equation}
        \label{recursive proof of safety policy improvement}
        \begin{aligned}
            &\gamma_h \min \left\{h(x), V_h^{\pi_h^k}\left(f\left(x, \pi_h^{k+1}(x)\right)\right)\right\} \\
            \geq& \gamma_h \min \left\{h(x), V_h^{\pi_h^k}\left(f\left(x, \left(\pi^{k+1}_{h,i_{1:n-1}}(x), \pi_{h,i_n}^k(x)\right)\right)\right)\right\} \\
            \geq& \gamma_h \min \left\{h(x), V_h^{\pi_h^k}\left(f\left(x, \left(\pi^{k+1}_{h,i_{1:n-2}}(x), \pi_{h,i_{n-1:n}}^k(x)\right)\right)\right)\right\} \\
            \geq& \ldots \\
            \geq& \gamma_h \min \left\{h(x), V_h^{\pi_h^k}\left(f\left(x, \pi_h^k(x)\right)\right)\right\}.
        \end{aligned}
    \end{equation}
    The inequalities in (\ref{recursive proof of safety policy improvement}) are obtained by recursively applying the definition of multi-agent safety policy improvement (Lines~\ref{first line of safety policy improvement} to~\ref{last line of safety policy improvement} in Algorithm~\ref{Multi-agent safety policy iteration}) in reverse.
    Now applying (\ref{recursive proof of safety policy improvement}) and the self-consistency condition~(\ref{Self-consistency Condition for safety value function}) to the definition of $V_h$, $\forall x_0\in\mathcal{X}$, the following relationship holds:
    \begin{equation*}
        \small
        \begin{aligned}
            &V_h^{\pi_h^k}\left(x_0\right) =\gamma_h \min \left\{h\left(x_0\right), V_h^{\pi_h^k}\left(f\left(x_0, \pi_h^k\left(x_0\right)\right)\right)\right\} \\
            \leq & \gamma_h \min \left\{h\left(x_0\right), V_h^{\pi_h^k}\left(f\left(x_0, \pi_h^{k+1}\left(x_0\right)\right)\right)\right\} \\
            =& \gamma_h \min \left\{h\left(x_0\right), \gamma_h \min \left\{h\left(x_1\right), V_h^{\pi_h^k}\left(f\left(x_1, \pi_h^k\left(x_1\right)\right)\right)\right\}\right\} \\
            \leq & \gamma_h \min \left\{h\left(x_0\right), \gamma_h \min \left\{h\left(x_1\right), V_h^{\pi_h^k}\left(f\left(x_1, \pi_h^{k+1}\left(x_1\right)\right)\right)\right)\right\} \\
            \leq & \cdots \\
            \leq & \gamma_h \min \left\{h\left(x_0\right), \gamma_h \min \left\{h\left(x_1\right), \gamma_h \min \left\{h\left(x_2\right), \cdots\right\}\right\}\right. \\
            =& \inf \left\{\gamma_h^{t+1} h\left(x_t\right)\right\} \\
            =& V_h^{\pi_h^{k+1}}\left(x_0\right),
            \end{aligned}
    \end{equation*}
    in which $x_{t+1}=f\left(x_{t}, \pi_h^{k+1}\left(x_{t}\right)\right)$ for $t\geq0$.

    Therefore, the sequence of safety value function on the joint safety policy $\left\{V_h^{\pi_h^k}\right\}$ generated by Algorithm~\ref{Multi-agent safety policy iteration} is non-decreasing. Moreover, $V_h^{\pi_h^k}$ is upper bounded by the global optimal safety value function $V_h^{\rm{opt}}$, which is the fixed point of the safety Bellman equation $V_h^{\rm{opt}}(x)=\gamma_h \min \left\{h(x), \max\limits_{u\in \mathcal{U}}\left\{V_h^{\rm{opt}}\left(f(x,u)\right)\right\}\right\}$ \cite{fisac2019bridging, li2024safe}.
    Consequently, the sequence $\left\{V_h^{\pi_h^k}\right\}$ converges. Once converged, performing the agent-by-agent sequential update will no longer alter the joint safety value function or safety policies.
    Let $\pi_h^*=(\pi_{h,1}^*, \pi_{h,2}^*, \cdots, \pi_{h,n}^*)$ denote the converged policies. $\forall i\in \mathcal{N}$ and $\forall x\in \mathcal{X}$, using the formulation from Line~\ref{first line of safety policy improvement} of Algorithm~\ref{Multi-agent safety policy iteration}, we have
    \begin{equation*}
        V_h^{(\pi_{h,i},\pi_{h,-i}^*)}(x)\leq V_h^{(\pi_{h,i}^*,\pi_{h,-i}^*)}(x)=V_h^{\pi_h^*}(x),
    \end{equation*}
    which holds for every individual safety policy $\pi_{h,i}\in\Pi_i$. It remains to verify Definition~\ref{Equilibria}. The feasible set is $\mathcal{F}=\Pi$. The shared objective is the expected safety value under the initial state distribution, i.e., $J(\pi_h)=\mathbb{E}_{x\sim d}\left\{V_h^{\pi_h}(x)\right\}$. Since the above inequality holds for every $x\in\mathcal{X}$, taking the expectation over $x\sim d$ gives $J(\pi_{h,i},\pi_{h,-i}^*)\leq J(\pi_h^*)$ for every agent $i\in\mathcal{N}$ and every $\pi_{h,i}\in\Pi_i$. Therefore, $\pi_h^*$ is a Nash equilibrium in the sense of Definition~\ref{Equilibria}.
\end{proof}

\begin{remark}
    Algorithm~\ref{Multi-agent safety policy iteration} can be understood as a method for solving an \emph{auxiliary} safety game $\left(\mathcal{N}, \mathcal{X}, \mathcal{U}, f, h, \gamma_h, d\right)$, which has the same agents and dynamics as $\mathcal{M}$ but takes the safety value function as the shared objective and does not involve the reward function $r$. It is introduced first since it provides the identification of CIS, paving the way for the subsequent constructions.
\end{remark}

\subsection{Constrained Optimization Problem Formulation}
\label{Problem Formulation}
The identification of CIS is a critical step in designing safe MARL algorithms, as it directly impacts the objective function. For states outside the CIS, the system can violate the constraints regardless of the control inputs. Therefore, it is meaningless to pursue rewards in these regions; instead, the policy should be optimized to improve the safety value function. This adjustment ensures that the system minimizes constraint violations. For states within the CIS, the system can safely pursue rewards, provided that the policies do not take any dangerous actions that might drive the system out of CIS.

In the single-agent setting, we can classify a state based on whether it is within the \textbf{maximal} CIS. However, this principle does not directly extend to the multi-agent case, as obtaining the maximal CIS is not feasible, in the sense that global optimality is traded off for computational scalability, as indicated in the previous subsection. Let $\pi_h^*$ denote a joint safety policy to which the multi-agent safety policy iteration converges. We refer to its CIS as a \textbf{maximal identifiable} CIS. It contains the CIS of every joint safety policy generated during the iteration, and by Theorem~\ref{Convergence of Multi-agent Safety Policy Iteration}, no agent can enlarge it by unilaterally changing its individual safety policy. Such a set is not unique in general, since $\pi_h^*$ depends on the initial joint safety policy and the order of sequential updates; the following formulation and analysis hold for any such $\pi_h^*$.
In this setting, there are three types of states: those within this maximal identifiable CIS, those outside the maximal CIS, and those within the maximal CIS but not within this maximal identifiable CIS. Since safety is paramount, we treat the objective for the third type of states as we do for states outside the maximal CIS—aiming to maximize the safety value function.
This is a reasonable consideration, as we lack a suitable mechanism to pursue rewards with safety guarantees for those states. Furthermore, we cannot determine whether this is even feasible, since the last two types of states are inseparable given a locally optimal safety value function.

Let $V^{\pi}$ denote the standard value function of a joint policy $\pi$, which is the discounted cumulative reward. Let $V_h^{*}$ denote the safety value function of $\pi_h^*$. Its zero-superlevel set $S_{\rm{c}}^*$ represents a maximal identifiable CIS, as defined in Definition~\ref{Controlled Invariant Set}. Let $\mathbbm{1}_{A}(a)$ denote the indicator function, i.e., $\mathbbm{1}_{A}(a)=1$ if $a\in A$, and $\mathbbm{1}_{A}(a)=0$ otherwise.
The following optimization problem is formulated for state-wise-constrained cooperative Markov games.

\begin{equation}
    \label{problem formulation}
    \begin{gathered}
        \max _\pi \mathop{\mathbb{E}}\limits_{x_0\sim d}\left\{V^{\pi}(x_0)\cdot \mathbbm{1}_{S_{\rm{c}}^*}(x_0)+V_h^{\pi}(x_0)\cdot \mathbbm{1}_{\mathcal{X}\setminus S_{\rm{c}}^*}(x_0)\right\} \\
        \begin{aligned}
            \text { s.t. \quad }&x_{t+1}=f\left(x_t, u_t\right), u_t= \pi\left( x_t \right), t\geq0,\\& V_h^*\left(f(x_t, \pi(x_t))\right) \geq 0, \ \forall x_t\in S_{\rm{c}}^*.
        \end{aligned}
    \end{gathered}
\end{equation}
As explained, this formulation has a two-fold objective. We optimize the standard value function only within this maximal identifiable CIS $S_{\rm{c}}^*$, as only in this region can reward maximization be achieved with a solid safety guarantee. Policy constraints are also enforced solely in this region, where the safety value function filters out any dangerous actions that can drive the system outside it.

Solving the safe MARL problem~(\ref{problem formulation}) is challenging. The constraints $V_h^*\left(f(x_t, \pi(x_t))\right) \geq 0$ apply to the joint policy, meaning that all agents face shared constraints when optimizing their individual policies and must coordinate their updates carefully to ensure joint constraint satisfaction. This is non-trivial, as an update by one agent changes the feasible directions for others, potentially causing conflicts.
We must avoid situations where one agent pursues excessive rewards during its update, leaving no safe actions available for others to maintain the joint policy’s ability to constrain the system state within the CIS.
Moreover, $V_h^*$ is unknown, as only the constraint function $h(x)$ is provided. Thus, the identification of CIS must be done simultaneously with constrained policy optimization. In the early stages of the algorithm, the known CIS is imperfect and potentially very small, so we need to design suitable updating mechanisms to address these scenarios. Our goal is to develop an iterative algorithm that not only gradually expands the known CIS but also ensures monotonic improvement of the objective function within these CISs.

\begin{remark}
    One might argue that the proposed multi-agent safety policy iteration could be used to \textit{precompute} a maximal identifiable CIS and then incorporate it as the constraint in the optimization problem~(\ref{problem formulation}). This approach is reasonable in the tabular setting; however, in practice, the corresponding deep MARL algorithm would perform poorly, as the policy and constraint function (i.e., the pretrained safety value function) are learned on different state-action distributions, leading to severe out-of-distribution issues.
\end{remark}

\subsection{Iteration Scheme}
\label{Iteration Scheme}
In this work, we present a multi-agent dual policy iteration approach to address the aforementioned challenges. Each agent maintains two individual policies: a task policy for solving the constructed problem~(\ref{problem formulation}) and a safety policy for collaboratively identifying the CIS. The safety policies are updated using the same approach as in multi-agent safety policy iteration. For task policies, we design a multi-agent task policy iteration scheme that alternates between joint task policy evaluation and multi-agent task policy improvement. In the latter step, the task policies are updated sequentially to improve the two-fold objective, subject to additional constraints within the current known CIS, as derived from the safety value function of the joint safety policy.

To present the overall algorithm, first we need the following definition.
\begin{definition}[invariant action set]
    \label{Invariant Action Set}
    The joint invariant action set for all agents at state $x$, under a joint safety policy $\pi_h$, is defined as
    \begin{equation*}
        \mathcal{U}^{\pi_h}(x)=\left\{u\in \mathcal{U} \mathrel{}\middle|\mathrel{} V_h^{\pi_h}\left(f(x, u)\right)\geq 0\right\}.
    \end{equation*}
    For agent $i\in \mathcal{N}$, given a state $x$, a joint safety policy $\pi_h$, and a joint action from all other agents $u_{-i}$, suppose that $\max \limits_{u_i\in \mathcal{U}_i} {V_h^{\pi_h}\left(f\left(x, \left(u_i, u_{-i}\right)\right)\right)}\geq 0$.
    Then the invariant action set for agent $i$ at state $x$ is defined as
    \begin{equation*}
        \mathcal{U}_{i}^{\pi_h}(x, u_{-i})=\left\{u_i\in \mathcal{U}_i \mathrel{}\middle|\mathrel{} \left(u_i, u_{-i}\right) \in \mathcal{U}^{\pi_h}(x)\right\}.
    \end{equation*}
\end{definition}

The invariant action set $\mathcal{U}_{i}^{\pi_h}(x, u_{-i})$ represents the feasible actions for agent $i$ at state $x$, given the actions $u_{-i}$ from other agents. If the agent contributes to the full joint action as $u=(u_i, u_{-i})$, and the system is recoverable from the next state $f(x, u)$ by switching to the safety policy $\pi_h$, then the action $u_i$ is considered safe. Consequently, the size of this set depends on the quality of the joint safety policy. The safer the policy $\pi_h$, the larger the invariant action set, giving the agent more flexibility in selecting actions and potentially leading to higher rewards.

The steps of the proposed multi-agent dual policy iteration scheme are summarized in Algorithm~\ref{Multi-agent dual policy iteration}. Note that for the joint task policy evaluation, the performance self-consistency operator $\mathcal{T}^{\pi}$ is defined as $\left[\mathcal{T}^{\pi} (V)\right](x)=r(x, \pi(x))+\gamma V\left(f(x,\pi(x))\right)$.

\begin{algorithm*}[htbp]
    \caption{Multi-agent dual policy iteration for state-wise-constrained cooperative Markov game}
    \label{Multi-agent dual policy iteration}
    \KwIn{Initial joint safety policy $\pi_h$, initial joint task policy $\pi\leftarrow \pi_h$, initial CIS $S_{\rm{c}}=\varnothing$.}
    \For{$m$ iterations}{
        \For{$k$ iterations}{
            \tcp{Joint safety policy evaluation}
            Same as Line~\ref{Safety policy evaluation} in Algorithm~\ref{Multi-agent safety policy iteration}.

            \tcp{Multi-agent safety policy improvement}
            Same as Lines~\ref{first line of full safety policy improvement} to~\ref{last line of full safety policy improvement} in Algorithm~\ref{Multi-agent safety policy iteration}.

            Update the joint safety policy as $\pi_h\leftarrow \pi_h^{\rm{new}}$.
        }

        \tcp{Joint safety policy evaluation}

        Same as Line~\ref{Safety policy evaluation} in Algorithm~\ref{Multi-agent safety policy iteration}.

        \tcp{Feasibility enforcement}

        Directly copy the joint safety policy to the joint task policy at states outside the current CIS $S_{\rm{c}}$:\label{first line of the failsafe mechanism}

        \For{each $x \in \mathcal{X}\setminus S_{\rm{c}}$}{
            $\pi(x)\leftarrow \pi_h(x)$
        }\label{last line of the failsafe mechanism}

        \tcp{Joint task policy evaluation}

        Solve $V^{\pi} = \mathcal{T}^{\pi}\left(V^{\pi}\right)$ for $V^{\pi}$.

        \tcp{Multi-agent task policy improvement}

        Calculate the latest CIS $S_{\rm{c}}^{\rm{new}}=S_{\rm{c}}^{\pi_h}=\left\{x \in \mathcal{X} \mathrel{}\middle|\mathrel{} V_h^{\pi_h}(x)\geq 0\right\}$.

        Randomly shuffle the order of agents $\mathcal{N}$ as $i_{1:n}$.

        Perform the following agent-by-agent sequential update at states inside the latest CIS $S_{\rm{c}}^{\rm{new}}$:

        \For{each $x \in S_{\rm{c}}^{\rm{new}}$}{

            {\small$\pi_{i_1}^{\rm{new}}(x) = \underset{u_{i_1} \in  \mathcal{U}_{i_1}^{\pi_h}\left(x, \pi_{i_{2:n}}(x)\right)}{\operatorname{argmax}} \left\{r\left(x, \left(u_{i_1}, \pi_{i_{2:n}}(x)\right)\right)+\gamma V^{\pi}\left(f\left(x, \left(u_{i_1}, \pi_{i_{2:n}}(x)\right)\right)\right)\right\}$}\label{first line of task policy improvement}

            {\small$\pi_{i_2}^{\rm{new}}(x) = \underset{u_{i_2} \in  \mathcal{U}_{i_2}^{\pi_h}\left(x, \left(\pi_{i_1}^{\rm{new}}(x), \pi_{i_{3:n}}(x)\right)\right)}{\operatorname{argmax}} \left\{ r\left(x, \left(\pi_{i_1}^{\rm{new}}(x), u_{i_2}, \pi_{i_{3:n}}(x)\right)\right)+\gamma V^{\pi}\left(f\left(x, \left(\pi_{i_1}^{\rm{new}}(x), u_{i_2}, \pi_{i_{3:n}}(x)\right)\right)\right)\right\}$}

            {\small$\ \quad\vdots$}

            {\small$\pi_{i_n}^{\rm{new}}(x) = \underset{u_{i_n} \in  \mathcal{U}_{i_n}^{\pi_h}\left(x, \pi_{i_{1:n-1}}^{\rm{new}}(x)\right)}{\operatorname{argmax}} \left\{ r\left(x, \left(\pi_{i_{1:n-1}}^{\rm{new}}(x), u_{i_n}\right)\right)+\gamma V^{\pi}\left(f\left(x, \left(\pi_{i_{1:n-1}}^{\rm{new}}(x), u_{i_n}\right)\right)\right)\right\}$}\label{last line of task policy improvement}
        }

        \For{each $x \in \mathcal{X}\setminus S_{\rm{c}}^{\rm{new}}$}{
            $\pi^{\rm{new}}(x)\leftarrow \pi_h(x)$
        }

        Update the joint task policy as $\pi\leftarrow \pi^{\rm{new}}$.

        Update the current CIS as $S_{\rm{c}}\leftarrow S_{\rm{c}}^{\rm{new}}$.
    }
\end{algorithm*}

The computational complexity of both safety policy improvement and task policy improvement is $O(Cn)$, where $C$ is the number of admissible actions for each agent and $n$ is the number of agents. This achieves the same level of computational cost as standard MARL algorithms proposed in \cite{bertsekas2021multiagent, bhattacharya2023multiagent, zhong2024heterogeneous, liu2024maximum}, while additionally providing safety guarantees. We will further show that multi-agent dual policy iteration converges to a generalized Nash equilibrium—a natural extension of the convergence to a Nash equilibrium in \cite{bertsekas2021multiagent, bhattacharya2023multiagent, zhong2024heterogeneous, liu2024maximum}.

\begin{remark}
    There is an extreme case in which $S_{\rm{c}}^*=\varnothing$. This indicates that the multi-agent safety policy iteration cannot find any safe region within the state space. In this scenario, the proposed formulation and multi-agent dual policy iteration will still work. In problem~(\ref{problem formulation}), the objective function reduces to the safety value function, and there are no constraints on the task policy, as they apply only within the CIS, which is now empty. Therefore, this turns into the same case as optimizing for the safety value function. In Algorithm~\ref{Multi-agent dual policy iteration}, calculating $S_{\rm{c}}^{\rm{new}}$ will always yield an empty result, so the constrained policy updates are never executed and the joint task policy is simply a copy of the joint safety policy at every state. Algorithm~\ref{Multi-agent dual policy iteration} will produce the same outcome as Algorithm~\ref{Multi-agent safety policy iteration}, as the problem essentially reduces to the same case. In the following analysis, we will only consider the non-trivial setting where a meaningful $S_{\rm{c}}^*$ can be identified.
\end{remark}

\begin{remark}
In Algorithm~\ref{Multi-agent dual policy iteration}, both the safety policy evaluation
(solve $V_h^{\pi_h}=\mathcal{T}_h^{\pi_h}(V_h^{\pi_h})$) and the task policy evaluation
(solve $V^\pi=\mathcal{T}^\pi(V^\pi)$) are policy evaluation problems under a fixed joint policy, and thus do not require enumerating every action in the joint action space $\mathcal{U}=\mathcal{U}_1\times\mathcal{U}_2\times\cdots\times\mathcal{U}_n$.
In the tabular setting, one can compute these fixed points by repeated application of the corresponding contraction mappings, i.e., by iterating $V_h\leftarrow\mathcal{T}_h^{\pi_h}(V_h)$ and $V\leftarrow\mathcal{T}^{\pi}(V)$ until convergence \cite{bertsekas2012dynamic, sutton2018reinforcement}. Therefore, policy evaluation scales polynomially (typically linearly) with the number of agents. Algorithms~\ref{Multi-agent safety policy iteration} and~\ref{Multi-agent dual policy iteration} avoid the exponential dependence on $n$ that arises in standard policy improvement (i.e., optimization over the joint action space) by performing agent-by-agent updates.
\end{remark}

To reveal important properties of the proposed algorithm, we introduce the following definition.

\begin{definition}[induced cooperative Markov game]
    Given a state-wise-constrained cooperative Markov game $\mathcal{M} = \left(\mathcal{N}, \mathcal{X}, \mathcal{U}, f, r, h, \gamma, \gamma_h, d\right)$, and a joint safety policy $\pi_h$ specifying a nonempty CIS $S_{\rm{c}}^{\pi_h}$, the induced cooperative Markov game is defined as $\mathcal{M}^{\pi_h} = \left(\mathcal{N}, S_{\rm{c}}^{\pi_h}, \mathcal{U}^{\pi_h}, f, r, \gamma, d_{\rm{c}}\right)$, where  $d_{\rm{c}}$ is the initial state distribution restricted to $S_{\rm{c}}^{\pi_h}$, and $\mathcal{U}^{\pi_h}$ denotes the state-dependent correspondence $x\mapsto\mathcal{U}^{\pi_h}(x)$ of admissible actions given in Definition~\ref{Invariant Action Set}.
\end{definition}

The induced game $\mathcal{M}^{\pi_h}$ is an auxiliary construction for solving problem~(\ref{problem formulation}). It encodes the state-wise safety requirement through state-dependent action spaces, namely the invariant action sets. As a result, no explicit state constraint is needed in its formulation: unlike the original game $\mathcal{M}$, the tuple of $\mathcal{M}^{\pi_h}$ does not include a constraint function $h(x)$. However, the invariant action sets of the individual agents remain coupled, since each of them depends on the actions of the other agents. We must therefore be careful with this definition. First, we need to show that the induced game is well-defined, i.e., given actions from the action space, the system will not leave the state space $S_{\rm{c}}^{\pi_h}$. Second, we need to show that the proposed algorithm can manage intertwined individual action spaces and that all agents have a nonempty action space to choose from at each state.

\begin{proposition}
    \label{forward invariance}
    Suppose a joint policy $\pi$ complies with the action space $\mathcal{U}^{\pi_h}$ within the CIS $S_{\rm{c}}^{\pi_h}$, i.e., $(u_1,u_2,\cdots,u_n)=\pi(x)$ satisfies $u_i\in \mathcal{U}_{i}^{\pi_h}(x,u_{-i})$ for all $i\in \mathcal{N}$ and all $x\in S_{\rm{c}}^{\pi_h}$. Then, starting from any state $x_0\in S_{\rm{c}}^{\pi_h}$, the system will never leave the CIS under the joint policy $\pi$.
\end{proposition}

\begin{proof}
    Given $x_0\in S_{\rm{c}}^{\pi_h}$, since the joint policy $\pi$ complies with the action space $\mathcal{U}^{\pi_h}$, we have $V_h^{\pi_h}\left(f(x_0,\pi(x_0))\right)\geq 0$ based on the definition of invariant action set. Therefore $V_h^{\pi_h}\left(x_1\right)\geq 0$, where $x_1$ denotes the next state in the trajectory. Based on the definition of CIS, we have $x_1\in S_{\rm{c}}^{\pi_h}$. By induction, the infinite-horizon trajectory starting from $x_0$ will always remain within the CIS.
\end{proof}

\begin{proposition}
    During the execution of Algorithm~\ref{Multi-agent dual policy iteration}, the constrained policy updates (Lines~\ref{first line of task policy improvement} to~\ref{last line of task policy improvement}) in task policy improvement will always be feasible, i.e., the invariant action sets being searched will never be empty.
\end{proposition}

\begin{proof}
    We will prove this proposition by induction. First, we demonstrate that the initial execution of constrained policy updates is feasible. Then, assuming feasibility in the $m$-th execution, we show that the $(m+1)$-th execution is also feasible.

    The first execution of constrained policy updates happens when a nonempty latest CIS $S_{\rm{c}}^{\rm{new}}$ is obtained. At this point, the current CIS $S_{\rm{c}}=\varnothing$, so Lines~\ref{first line of the failsafe mechanism} to~\ref{last line of the failsafe mechanism} are executed for all $x\in \mathcal{X}$, making the joint task policy identical to the joint safety policy. At state $x\in S_{\rm{c}}^{\rm{new}}$, for agent $i_1$, we have $V_h^{\pi_h}\left(f\left(x, \left(\pi_{h,i_1}(x),\pi_{-i_1}(x)\right)\right)\right)\geq0$ since $V_h^{\pi_h}(f(x, \pi_h(x)))\geq0$ and $\pi_{-i_1}(x)=\pi_{h,-i_1}(x)$. Thus, there exists an action $u_{i_1}=\pi_{h,i_1}(x)\in \mathcal{U}_{i_1}^{\pi_h}(x, u_{-i_1})$ that agent $i_1$ can choose. For agent $i_2$, we have $V_h^{\pi_h}\left(f\left(x, \left(\pi_{i_1}^{\rm{new}}(x),\pi_{h,i_2}(x),\pi_{i_{3:n}}(x)\right)\right)\right)\geq0$ since $V_h^{\pi_h}(f(x, \left(\pi_{i_1}^{\rm{new}}(x),\pi_{-i_1}(x)\right)))\geq0$ (the update of agent $i_1$  follows the constraints of $\mathcal{U}_{i_1}^{\pi_h}(x, u_{-i_1})$) and $\pi_{i_{3:n}}(x)=\pi_{h,i_{3:n}}(x)$. Therefore, there exists an action $u_{i_2}=\pi_{h,i_2}(x)\in \mathcal{U}_{i_2}^{\pi_h}(x, \left(\pi_{i_1}^{\rm{new}}(x),\pi_{i_{3:n}}(x)\right))$ that agent $i_2$ can choose. This logic extends to all agents, which shows that the initial execution of constrained policy updates is feasible.

    Assume that the $m$-th execution of constrained policy updates is feasible, with the corresponding joint safety policy denoted as $\bar{\pi}_h$. Since the safety value function is non-decreasing in the multi-agent safety policy iteration, we have $V_h^{\pi_h}(x)\geq V_h^{\bar{\pi}_h}(x)$ for all $x\in \mathcal{X}$. Therefore, $S_{\rm{c}}\subseteq S_{\rm{c}}^{\rm{new}}$. In Lines~\ref{first line of the failsafe mechanism} to~\ref{last line of the failsafe mechanism}, for states belonging to $S_{\rm{c}}^{\rm{new}}\setminus S_{\rm{c}}$, the joint task policy is also directly copied from the joint safety policy. In the constrained policy updates for these states, the proof of feasibility is similar to the case of the initial execution. This leaves the case of $x\in S_{\rm{c}}$ to address. For agent $i_1$, at state $x\in S_{\rm{c}}$, given that the $m$-th execution of constrained policy updates is feasible, we have $V_h^{\bar{\pi}_h}\left(f\left(x, \left(\pi_{i_1}(x), \pi_{-i_1}(x)\right)\right)\right)\geq0$. Since the safety value function is non-decreasing ($V_h^{\pi_h}(x)\geq V_h^{\bar{\pi}_h}(x)$), it follows that $V_h^{\pi_h}\left(f\left(x, \left(\pi_{i_1}(x), \pi_{-i_1}(x)\right)\right)\right)\geq0$ as well. Therefore, there exists an action $u_{i_1}=\pi_{i_1}(x)\in \mathcal{U}_{i_1}^{\pi_h}(x, \pi_{-i_1}(x))$ that agent $i_1$ can choose. For agent $i_2$, it follows that $V_h^{\pi_h}\left(f\left(x, \left(\pi_{i_1}^{\rm{new}}(x), \pi_{i_2}(x), \pi_{i_{3:n}}(x)\right)\right)\right)\geq0$. Thus, there exists an action $u_{i_2}=\pi_{i_2}(x)\in \mathcal{U}_{i_2}^{\pi_h}(x, \left(\pi_{i_1}^{\rm{new}}(x), \pi_{i_{3:n}}(x)\right))$ that agent $i_2$ can select. This reasoning applies to all agents, confirming that the $(m+1)$-th execution of constrained policy updates is feasible.
\end{proof}

Now we are ready to present the main result on the convergence of Algorithm~\ref{Multi-agent dual policy iteration}. Unlike the safety policies, the task policies are subject to a constraint shared by all agents. Since this constraint is imposed on the joint action, whether an individual task policy is feasible depends on the task policies of the other agents, and the relevant notion is therefore the generalized Nash equilibrium. The following theorem shows that Algorithm~\ref{Multi-agent dual policy iteration} converges to such an equilibrium, i.e., to a joint task policy that is person-by-person optimal subject to the shared constraint.

\begin{theorem}[convergence of multi-agent dual policy iteration]
    \label{Convergence of Multi-agent Dual Policy Iteration}
    Algorithm~\ref{Multi-agent dual policy iteration} converges to a generalized Nash equilibrium in the induced cooperative Markov game $\mathcal{M}^{\pi_h^*}$ on the standard value function. Furthermore, this equilibrium is also a generalized Nash equilibrium for the constructed problem~(\ref{problem formulation}) in the original state-wise-constrained cooperative Markov game $\mathcal{M}$ with the two-fold objective.
\end{theorem}

\begin{proof}
    There are two threads in multi-agent dual policy iteration. The first thread is the multi-agent safety policy iteration, which operates independently of the task policy. Its convergence is proven in Theorem~\ref{Convergence of Multi-agent Safety Policy Iteration}: the joint safety policy converges to a Nash equilibrium on the safety value function, corresponding to a maximal identifiable CIS $S_{\rm{c}}^*$. This leaves us with the second thread, the multi-agent task policy iteration, which is subject to the constraints provided by the first thread throughout the iteration process.

    Given the joint safety policy $\pi_h$, we define the induced cooperative Markov game $\mathcal{M}^{\pi_h} = \left(\mathcal{N}, S_{\rm{c}}^{\rm{new}}, \mathcal{U}^{\pi_h}, f, r, \gamma, d_{\rm{c}}\right)$, where we use $S_{\rm{c}}^{\rm{new}}$ interchangeably with $S_{\rm{c}}^{\pi_h}$ to better align with the notation in Algorithm~\ref{Multi-agent dual policy iteration}. In this induced game, we show that the standard value function of the joint task policy $\pi$ is non-decreasing.
    Suppose $i_{1:n}$ is the optimization order of agents in the $m$-th iteration of task policy. $\forall x\in S_{\rm{c}}^{\rm{new}}$,  we have
    \begin{equation}
        \label{recursive proof of task policy improvement}
        \begin{aligned}
            &r\left(x,\pi^{m+1}(x)\right)+\gamma V^{\pi^m}\left(f\left(x, \pi^{m+1}(x)\right)\right) \\
            \geq& r\left(x, \left(\pi^{m+1}_{i_{1:n-1}}(x), \pi_{i_{n}}^m(x)\right)\right)\\&+\gamma V^{\pi^m}\left(f\left(x, \left(\pi^{m+1}_{i_{1:n-1}}(x), \pi_{i_{n}}^m(x)\right)\right)\right) \\
            \geq& r\left(x, \left(\pi^{m+1}_{i_{1:n-2}}(x), \pi_{i_{n-1:n}}^m(x)\right)\right)\\&+\gamma V^{\pi^m}\left(f\left(x, \left(\pi^{m+1}_{i_{1:n-2}}(x), \pi_{i_{n-1:n}}^m(x)\right)\right)\right) \\
            \geq& \ldots \\
            \geq&r\left(x,\pi^{m}(x)\right)+\gamma V^{\pi^m}\left(f\left(x, \pi^{m}(x)\right)\right) .
        \end{aligned}
    \end{equation}
    The inequalities in (\ref{recursive proof of task policy improvement}) are obtained by applying the definition of constrained policy updates (Lines~\ref{first line of task policy improvement} to~\ref{last line of task policy improvement}) backwards. Then, utilizing (\ref{recursive proof of task policy improvement}) and the self-consistency condition of the standard value function along the infinite-horizon trajectory, $\forall x_0\in S_{\rm{c}}^{\rm{new}}$, we have
    \begin{equation*}
        \begin{aligned}
            &V^{\pi^m}\left(x_0\right) = r\left(x_0,\pi^{m}(x_0)\right)+\gamma V^{\pi^m}\left(f\left(x_0, \pi^{m}(x_0)\right)\right) \\
            \leq& r\left(x_0,\pi^{m+1}(x_0)\right)+\gamma V^{\pi^m}\left(f\left(x_0, \pi^{m+1}(x_0)\right)\right) \\
            = & r\left(x_0,\pi^{m+1}(x_0)\right)\\&+\gamma \left(r\left(x_1, \pi^{m}(x_1)\right)+\gamma V^{\pi^m}\left(f\left(x_1, \pi^{m}(x_1)\right)\right) \right) \\
            \leq& r\left(x_0,\pi^{m+1}(x_0)\right)\\&+\gamma \left(r\left(x_1, \pi^{m+1}(x_1)\right)+\gamma V^{\pi^m}\left(f\left(x_1, \pi^{m+1}(x_1)\right)\right) \right) \\
            \leq& \ldots \\
            \leq & \sum_{t=0}^{\infty} \gamma^t r\left(x_t, \pi^{m+1}\left(x_t\right)\right)\\
            =& V^{\pi^{m+1}}\left(x_0\right).
        \end{aligned}
    \end{equation*}
    in which $x_{t+1}=f\left(x_{t}, \pi^{m+1}\left(x_{t}\right)\right)$ for $t\geq0$. The last inequality holds since the term remaining after $T$ steps, $\gamma^{T}V^{\pi^m}(x_T)$, vanishes as $T\rightarrow\infty$.
    Since $\mathcal{M}^{\pi_h}$ is well-defined, the standard value function of any joint policy is upper bounded by the globally optimal value function $V^{\rm{opt}}$ in the induced game, which is the fixed point of the Bellman equation $V^{\rm{opt}}(x)=\max \limits_{u\in \mathcal{U}^{\pi_h}(x)}\left\{r(x,u)+\gamma V^{\rm{opt}}\left(f(x,u)\right)\right\}$. Therefore, the standard value function sequence generated by the multi-agent task policy iteration is both non-decreasing and upper bounded on a fixed induced game $\mathcal{M}^{\pi_h}$.
    As the iteration proceeds, the joint safety policy converges  to $\pi_h^*$, and the CIS converges to $S_{\rm{c}}^*$. Consequently, the task policy will converge on the induced game $\mathcal{M}^{\pi_h^*}$.

    The convergence implies that the joint task policy $\pi$ will remain the same when performing the constrained policy updates. Denote the converged joint task policy as $\pi^*=(\pi_1^*, \pi_2^*, \cdots, \pi_n^*)$. We now show that $\pi^*$ is a generalized Nash equilibrium in the sense of Definition~\ref{Equilibria}. For both claims of the theorem, the set of feasible joint policies is
    \begin{equation*}
        \mathcal{F}=\left\{\pi\in\Pi \mathrel{}\middle|\mathrel{} \pi(x)\in\mathcal{U}^{\pi_h^*}(x), \ \forall x\in S_{\rm{c}}^*\right\}.
    \end{equation*}
    This set consists of the joint policies that satisfy the constraint of problem~(\ref{problem formulation}), or equivalently, that comply with the action space of the induced game $\mathcal{M}^{\pi_h^*}$. Moreover, $\pi^*\in\mathcal{F}$ since the constrained policy updates select actions from the invariant action sets.
    $\forall i\in \mathcal{N}$ and $\forall x\in S_{\rm{c}}^*$, using the formulation from Lines~\ref{first line of task policy improvement} to~\ref{last line of task policy improvement} of Algorithm~\ref{Multi-agent dual policy iteration} and the definition of invariant action set, we have
    \begin{equation*}
        V^{(\pi_{i},\pi_{-i}^*)}(x)\leq V^{(\pi_{i}^*,\pi_{-i}^*)}(x)=V^{\pi^*}(x),
    \end{equation*}
    which holds for every individual task policy $\pi_i\in\Pi_i$ satisfying $(\pi_i,\pi_{-i}^*)\in\mathcal{F}$.

    In the induced game $\mathcal{M}^{\pi_h^*}$, the shared objective is the expected standard value under the restricted initial state distribution, i.e., $J(\pi)=\mathbb{E}_{x\sim d_{\rm{c}}}\left\{V^{\pi}(x)\right\}$. Since the above inequality holds for every $x\in S_{\rm{c}}^*$, taking the expectation over $x\sim d_{\rm{c}}$ gives $J(\pi_i,\pi_{-i}^*)\leq J(\pi^*)$ for every agent $i\in\mathcal{N}$ and every $\pi_i\in\Pi_i$ satisfying $(\pi_i,\pi_{-i}^*)\in\mathcal{F}$. Therefore, $\pi^*$ is a generalized Nash equilibrium in the induced game.

    For problem~(\ref{problem formulation}), the shared objective $J$ is its two-fold objective. Consider any agent $i\in\mathcal{N}$ and any $\pi_i\in\Pi_i$ satisfying $(\pi_i,\pi_{-i}^*)\in\mathcal{F}$. For $x_0\in S_{\rm{c}}^*$, we have $V^{(\pi_i,\pi_{-i}^*)}(x_0)\leq V^{\pi^*}(x_0)$ as shown above. For $x_0\notin S_{\rm{c}}^*$, we need to show that $V_h^{(\pi_i,\pi_{-i}^*)}(x_0)\leq V_h^{\pi^*}(x_0)$. Let $\tilde{\pi}_{h,i}$ denote the individual safety policy that equals $\pi_i$ outside $S_{\rm{c}}^*$ and $\pi_{h,i}^*$ inside $S_{\rm{c}}^*$. Since $\pi^*$ coincides with $\pi_h^*$ outside $S_{\rm{c}}^*$, the joint policies $(\pi_i,\pi_{-i}^*)$ and $(\tilde{\pi}_{h,i},\pi_{h,-i}^*)$ generate the same trajectory from $x_0$ until it enters $S_{\rm{c}}^*$. Both of them belong to $\mathcal{F}$ (the latter equals $\pi_h^*$ inside $S_{\rm{c}}^*$), so by Proposition~\ref{forward invariance} the state remains in $S_{\rm{c}}^*$ afterwards, where all terms $\gamma_h^{t+1}h(x_t)$ in Definition~\ref{definition of safety value} are nonnegative. The proof of Theorem~\ref{Convergence of Multi-agent Safety Policy Iteration} shows that $V_h^{(\tilde{\pi}_{h,i},\pi_{h,-i}^*)}(x_0)\leq V_h^{\pi_h^*}(x_0)$, which is negative since $x_0\notin S_{\rm{c}}^*$. For both joint policies, the infimum in Definition~\ref{definition of safety value} is therefore attained on the common part of the two trajectories, which yields
    \begin{equation*}
        V_h^{(\pi_i,\pi_{-i}^*)}(x_0)=V_h^{(\tilde{\pi}_{h,i},\pi_{h,-i}^*)}(x_0)\leq V_h^{\pi_h^*}(x_0)=V_h^{\pi^*}(x_0),
    \end{equation*}
    where the last equality follows from the same argument with $\pi_i=\pi_i^*$. Combining the two cases and taking the expectation over $x_0\sim d$ gives $J(\pi_i,\pi_{-i}^*)\leq J(\pi^*)$. Therefore, $\pi^*$ is also a generalized Nash equilibrium for problem~(\ref{problem formulation}) in the original state-wise-constrained cooperative Markov game $\mathcal{M}$.
\end{proof}

One might wonder: since in Algorithm~\ref{Multi-agent dual policy iteration}, we only explicitly identify the region where the joint safety policy $\pi_h$ is safe (i.e., the CIS $S_{\rm{c}}^{\pi_h}$), does the outcome of the constrained policy updates truly make the joint task policy $\pi$ safe? In other words, what is the safe region of the joint task policy $\pi$ (i.e., $S_{\rm{c}}^{\pi}$)?
The following proposition provides the answer.

\begin{proposition}
    The CIS $S_{\rm{c}}^{\pi}$ associated with the joint task policy $\pi$ in Algorithm~\ref{Multi-agent dual policy iteration} is non-shrinking throughout the iterative process and will converge to the same maximal identifiable CIS as the joint safety policy $\pi_h$, denoted by $S_{\rm{c}}^*$.
\end{proposition}

\begin{proof}
    After the constrained policy updates in Algorithm~\ref{Multi-agent dual policy iteration}, the joint task policy $\pi$ complies with the action space $\mathcal{U}^{\pi_h}$ within the CIS $S_{\rm{c}}^{\pi_h}$, i.e., $(u_1,u_2,\cdots,u_n)=\pi(x)$ satisfies $u_i\in \mathcal{U}_{i}^{\pi_h}(x,u_{-i})$ for all $i \in \mathcal{N}$ and $x \in S_{\rm{c}}^{\pi_h}$. By Proposition~\ref{forward invariance}, the system will thus remain within the CIS under $\pi$.
    Since the CIS is a subset of the constraint set $S_h$, we have $h(x) \geq 0$ for any trajectory starting from $x \in S_{\rm{c}}^{\pi_h}$ and driven by $\pi$.
    Using the definition of safety value function, we have $V_h^{\pi}(x)=\inf \limits_{t \in \mathbb{N}}\left\{\gamma_h^{t+1}h\left(x_t\right)\right\}$, where $\left\{x_{t}\right\}$ represents the state trajectory driven by $\pi$ starting from $x$.
    Consequently, for all $x \in S_{\rm{c}}^{\pi_h}$, $V_h^{\pi}(x) \geq 0$, indicating $x \in S_{\rm{c}}^{\pi}$. For $x \notin S_{\rm{c}}^{\pi_h}$, the joint task policy $\pi$ copies the joint safety policy $\pi_h$ (which cannot keep the system indefinitely safe in these states), so $x \notin S_{\rm{c}}^{\pi}$ as well. This indicates that the CIS $S_{\rm{c}}^{\pi}$ is equivalent to $S_{\rm{c}}^{\pi_h}$. The conclusion follows from the fact that $S_{\rm{c}}^{\pi_h}$ is non-shrinking and converges to a maximal identifiable CIS $S_{\rm{c}}^*$.
\end{proof}

\begin{remark}
    The results of this section extend naturally to stochastic dynamics $x_{t+1}\sim P(\cdot\,|\,x_t,u_t)$, where the constraint $h(x_t)\geq 0$ is required to hold for all $t$ almost surely. The key is to define the safety value function $V_h^{\pi_h}(x)$ as the infimum of $\gamma_h^{t+1}h(x_t)$ over all time steps $t$ and all trajectories that can occur from $x$ under $\pi_h$, i.e., with $x_0=x$ and $P(x_{t+1}\,|\,x_t,\pi_h(x_t))>0$ for all $t$. Then $V_h^{\pi_h}(x)\geq 0$ holds if and only if the system, starting from $x$, remains safe under $\pi_h$ almost surely. In Algorithms~\ref{Multi-agent safety policy iteration} and~\ref{Multi-agent dual policy iteration} and in Definition~\ref{Invariant Action Set}, $V_h^{\pi_h}\left(f(x,u)\right)$ is replaced by $\min_{x^{\prime}:P(x^{\prime}|x,u)>0}V_h^{\pi_h}(x^{\prime})$ and $V^{\pi}\left(f(x,u)\right)$ by $\mathbb{E}_{x^{\prime}\sim P(\cdot|x,u)}\left\{V^{\pi}(x^{\prime})\right\}$, and the proofs follow the same steps. The deterministic case is recovered when $P(\cdot\,|\,x,u)$ is a point mass at $f(x,u)$.
\end{remark}

\section{Multi-Agent Dual Actor-Critic}

In this section, we present a safe MARL algorithm, multi-agent dual actor-critic (MADAC), for practical deployment in complex high-dimensional systems, which approximates the multi-agent dual policy iteration within a deep RL framework. To facilitate exploration in high-dimensional continuous spaces and enhance performance, we integrate our algorithm with HASAC \cite{liu2024maximum}, the multi-agent counterpart of the single-agent SAC algorithm \cite{haarnoja2018soft}. The overall procedures of MADAC are summarized in Algorithm~\ref{MADAC}, where the comments indicate the corresponding steps of Algorithm~\ref{Multi-agent dual policy iteration}.

\begin{algorithm}[t]
    \caption{Multi-Agent Dual Actor-Critic}
    \label{MADAC}
    \KwIn{network parameters $\omega_j$, $\psi_j$, $\hat{\psi}_j\leftarrow\psi_j$, $\hat{\omega}_j\leftarrow\omega_j$, $j\in\left\{1,2\right\}$, $\theta_i$, $\phi_i$, $i\in\mathcal{N}$, temperature $\alpha$, learning rate $\beta$, target smoothing coefficient $\tau$, replay buffer $\mathcal{D}\leftarrow \varnothing$.}
    \For{each iteration}{
        \For{each system step}{
            Sample control input $u_t\sim \prod_{i=1}^{n}\pi_i\left(x_t;\theta_i\right)$;

            Observe next state $x_{t+1}$, reward $r_t$, constraint value $h_{t+1}$;

            Store transition $\mathcal{D} \leftarrow \mathcal{D} \cup\left\{\left(x_t, u_t, r_t, h_t, h_{t+1}, x_{t+1}\right)\right\}$.
        }

        \For{each gradient step}{
            Sample a batch of data from $\mathcal{D}$;

            \tcp{Joint safety policy evaluation}
            Update safety value functions $\psi_j \leftarrow \psi_j-\beta \nabla_{\psi_j} L_{Q_h}(\psi_j)$ for $j\in\left\{1,2\right\}$;

            \tcp{Multi-agent safety policy improvement}
            Randomly shuffle the order of agents $\mathcal{N}$ as $i_{1:n}$.

            \For {each agent $i_k$}{
                Update safety policy $\phi_{i_k} \leftarrow \phi_{i_k} - \beta \nabla_{\phi_{i_k}} L_{\pi_{h,i_k}}\left(\phi_{i_k}\right)$.
            }

            \tcp{Feasibility enforcement}
            Update task policies $\theta_{i} \leftarrow \theta_{i} - \beta \nabla_{\theta_{i}} L_{\pi_{i}}^{\rm{fe}}\left(\theta_{i}\right)$ for $i\in\mathcal{N}$;

            \tcp{Joint task policy evaluation}
            Update value functions $\omega_j \leftarrow \omega_j-\beta \nabla_{\omega_j} L_{Q}(\omega_j)$ for $j\in\left\{1,2\right\}$;

            \tcp{Multi-agent task policy improvement}
            Randomly shuffle the order of agents $\mathcal{N}$ as $i_{1:n}$.

            \For {each agent $i_k$}{

                Update task policy $\theta_{i_k} \leftarrow \theta_{i_k} - \beta \nabla_{\theta_{i_k}} L_{\pi_{i_k}}\left(\theta_{i_k}\right)$;

                Update Lagrange multiplier $\lambda_{i_k} \leftarrow \max \left\{0, \lambda_{i_k} - \beta \nabla_{\lambda_{i_k}} L_{\lambda_{i_k}}\right\}$.
            }

            Update target networks $\hat{\psi}_j \leftarrow \tau \psi_j+(1-\tau) \hat{\psi}_j$, $\hat{\omega}_j \leftarrow \tau \omega_j+(1-\tau) \hat{\omega}_j$ for $j\in\left\{1,2\right\}$.\label{target network update}
        }
    }
\end{algorithm}

\begin{figure}[htbp]
    \centering
    \subfloat[HalfCheetah]{
        \label{HalfCheetah pic}
        \includegraphics[width=1.0in]{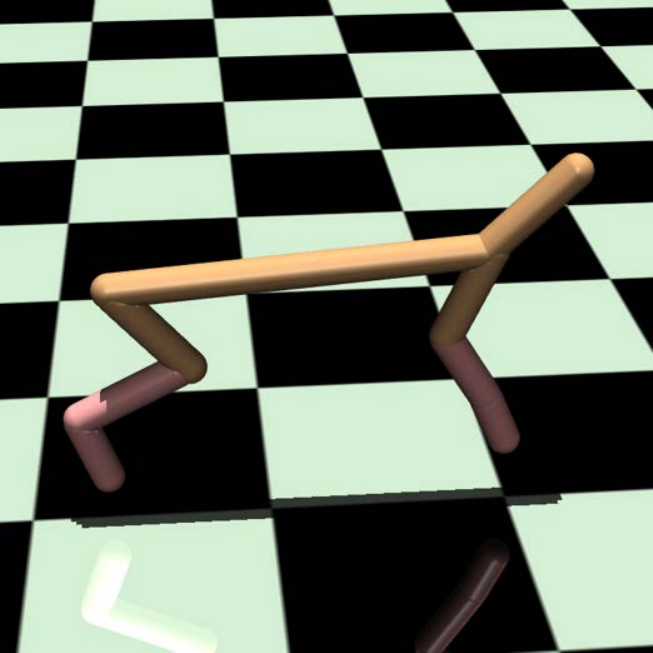}
    }
    \subfloat[Walker2D]{
        \label{Walker2D pic}
        \includegraphics[width=1.0in]{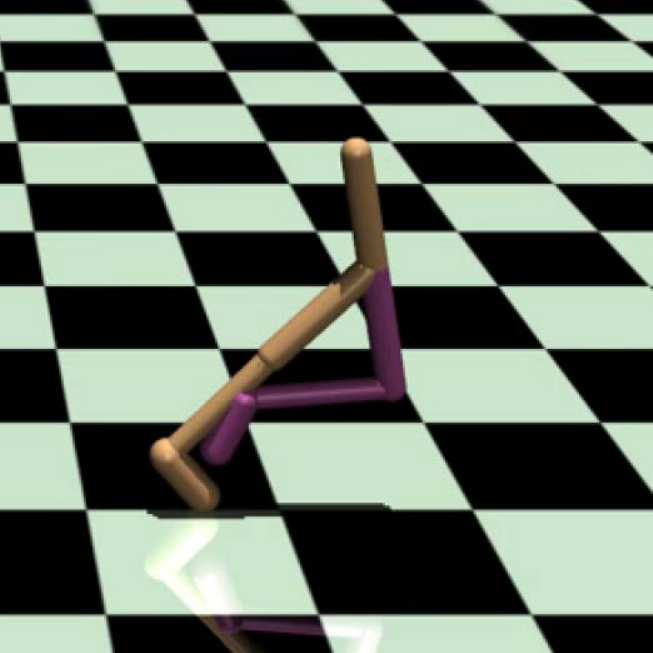}
    }
    \subfloat[Ant]{
        \label{Ant pic}
        \includegraphics[width=1.0in]{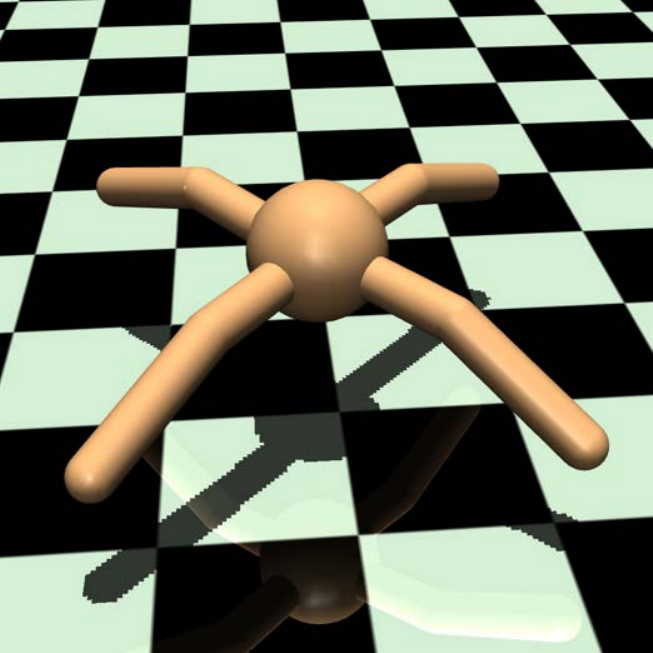}
    }
    \caption{Snapshots of three robots.}
\end{figure}

For agent $i$, the individual task policy network is denoted by $\pi_i\left(x;\theta_i\right)$, and the individual safety policy network by $\pi_{h,i}\left(x;\phi_i\right)$. Following the double Q-network design from SAC and HASAC, our algorithm employs two value networks and two safety value networks. The value networks are denoted as $Q\left(x,u;\omega_1\right)$ and $Q\left(x,u;\omega_2\right)$, and the safety value networks as $Q_{h}\left(x,u;\psi_1\right)$ and $Q_{h}\left(x,u;\psi_2\right)$. Here, $\theta_i$ and $\phi_i$ are the parameters of the two policy networks of agent $i$, and $\omega_j$ and $\psi_j$, $j\in\left\{1,2\right\}$, are those of the value networks and the safety value networks, respectively. Each safety value network approximates the safety value of the successor state under the current joint safety policy $\pi_h$, i.e., $Q_h(x,u;\psi_j)\approx V_h^{\pi_h}\left(f(x,u)\right)$. By Definition~\ref{Invariant Action Set}, the condition $u\in\mathcal{U}^{\pi_h}(x)$ thus becomes the constraint $\min_{j\in\left\{1,2\right\}}\left\{Q_h(x,u;\psi_j)\right\}\geq 0$. Since the invariant action sets cannot be computed or traversed in the continuous case, we introduce a Lagrange multiplier $\lambda_i$ for this constraint to facilitate constrained policy optimization.

Given a sample collection $\mathcal{D}$, the loss functions of value networks are
\begin{equation*}
    L_Q\left(\omega_{j}\right)=\mathbb{E}_{\left(x, u, r, x^{\prime}\right) \sim \mathcal{D}}\left\{\left(Q(x,u;\omega_{j})-\hat{Q}\right)^2\right\},
\end{equation*}
where $j\in\left\{1,2\right\}$ and $\left(x, u, r, x^{\prime}\right) \sim \mathcal{D}$ denotes a transition sampled uniformly from $\mathcal{D}$. The target value $\hat{Q}$ is computed as
\begin{equation*}
    \hat{Q}=r(x, u)+\gamma\left( \min_{j\in\left\{1,2\right\}}\left\{Q\left(x^{\prime},u^{\prime};\hat{\omega}_j\right)\right\}-\alpha \log \pi\left(u^{\prime}|x^{\prime};\theta\right)\right),
\end{equation*}
in which $\hat{\omega}_j$ represents the target network parameters, which slowly track $\omega_j$ (Line~\ref{target network update} of Algorithm~\ref{MADAC}), $\alpha$ is the temperature, and $\pi\left(u^{\prime}|x^{\prime};\theta\right)$ is given by
\begin{equation*}
    \pi\left(u^{\prime}|x^{\prime};\theta\right)=\prod_{i=1}^{n}\pi_i\left(u_i^{\prime}\mid x^{\prime};\theta_i\right),
\end{equation*}
where $u_i^{\prime}\sim\pi_i\left(\cdot|x^{\prime};\theta_i\right)$. The loss functions of safety value networks are
\begin{equation*}
    L_{Q_h}(\psi_{j}) = \mathbb{E}_{(x,u,h^{\prime},x^{\prime}) \sim \mathcal{D}}\left\{\left(Q_h(x,u ; \psi_{j})-\hat{Q}_h\right)^2\right\},
\end{equation*}
where $j\in\left\{1,2\right\}$. The target safety value $\hat{Q}_h$ is
\begin{equation*}
    \hat{Q}_h = \gamma_h \min \left\{h(x^{\prime}),  \min_{j\in\left\{1,2\right\}} \left\{Q_h\left(x^{\prime}, u^{\prime};\hat{\psi}_j\right)\right\}\right\},
\end{equation*}
in which $u^{\prime}=\left(\pi_{h,1}(x^{\prime};\phi_1),\cdots,\pi_{h,n}(x^{\prime};\phi_n)\right)$, and $\hat{\psi}_j$ represents the target network parameters. This target approximates the right-hand side of (\ref{Self-consistency Condition for safety value function}) at the successor state $x^{\prime}$.

\begin{figure*}[htbp]
    \centering
    \includegraphics[height=4.25in]{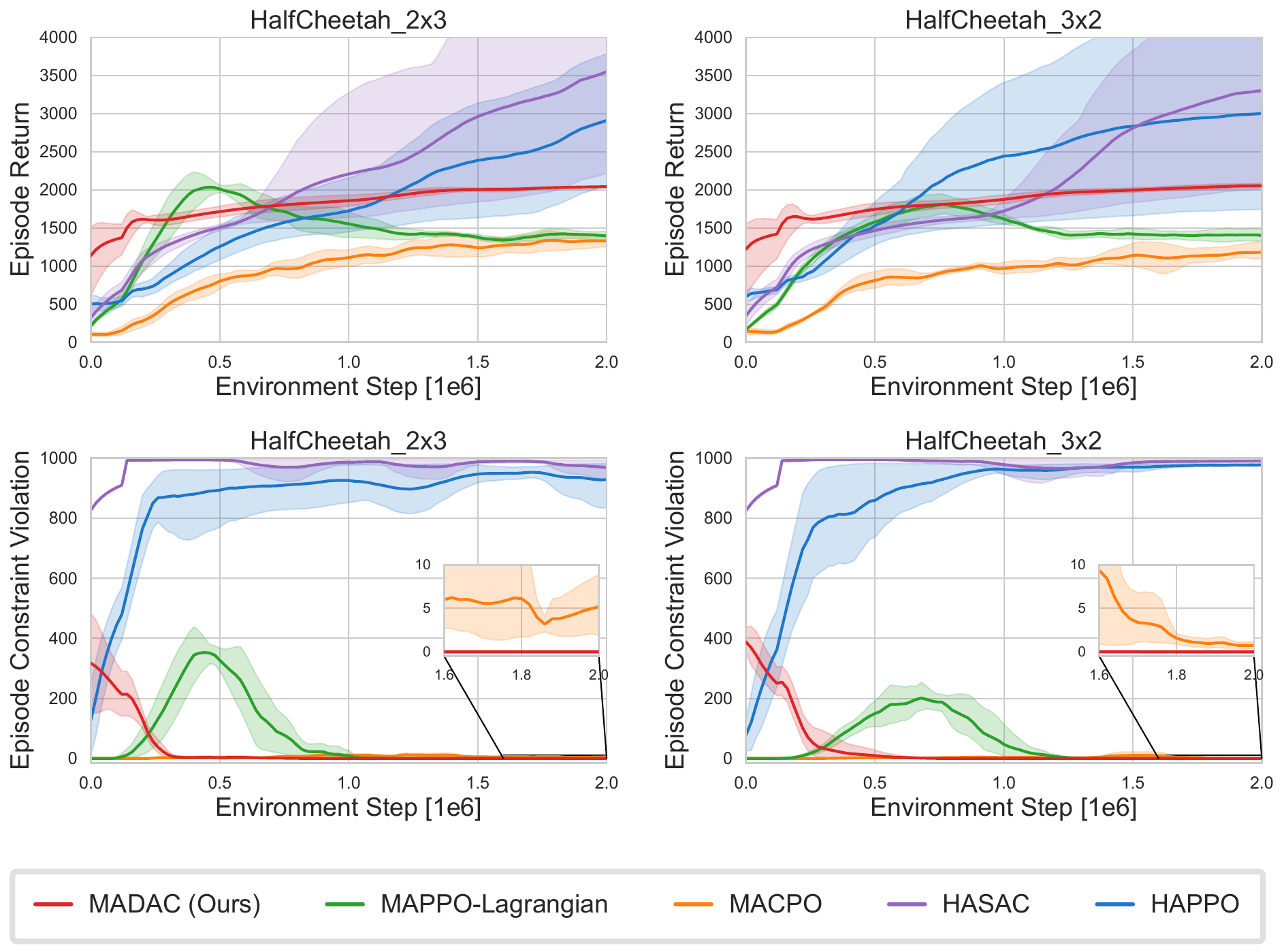}
    \caption{Training curves on HalfCheetah environments. While standard MARL baselines (HASAC and HAPPO) achieve higher rewards, they violate the constraints heavily. MADAC outperforms safe MARL baselines (MACPO and MAPPO-Lagrangian) by achieving significantly higher rewards, while maintaining equal or superior compliance with safety constraints.}
    \label{HalfCheetah_Results}
\end{figure*}

\begin{figure*}[htbp]
    \centering
    \includegraphics[height=4.25in]{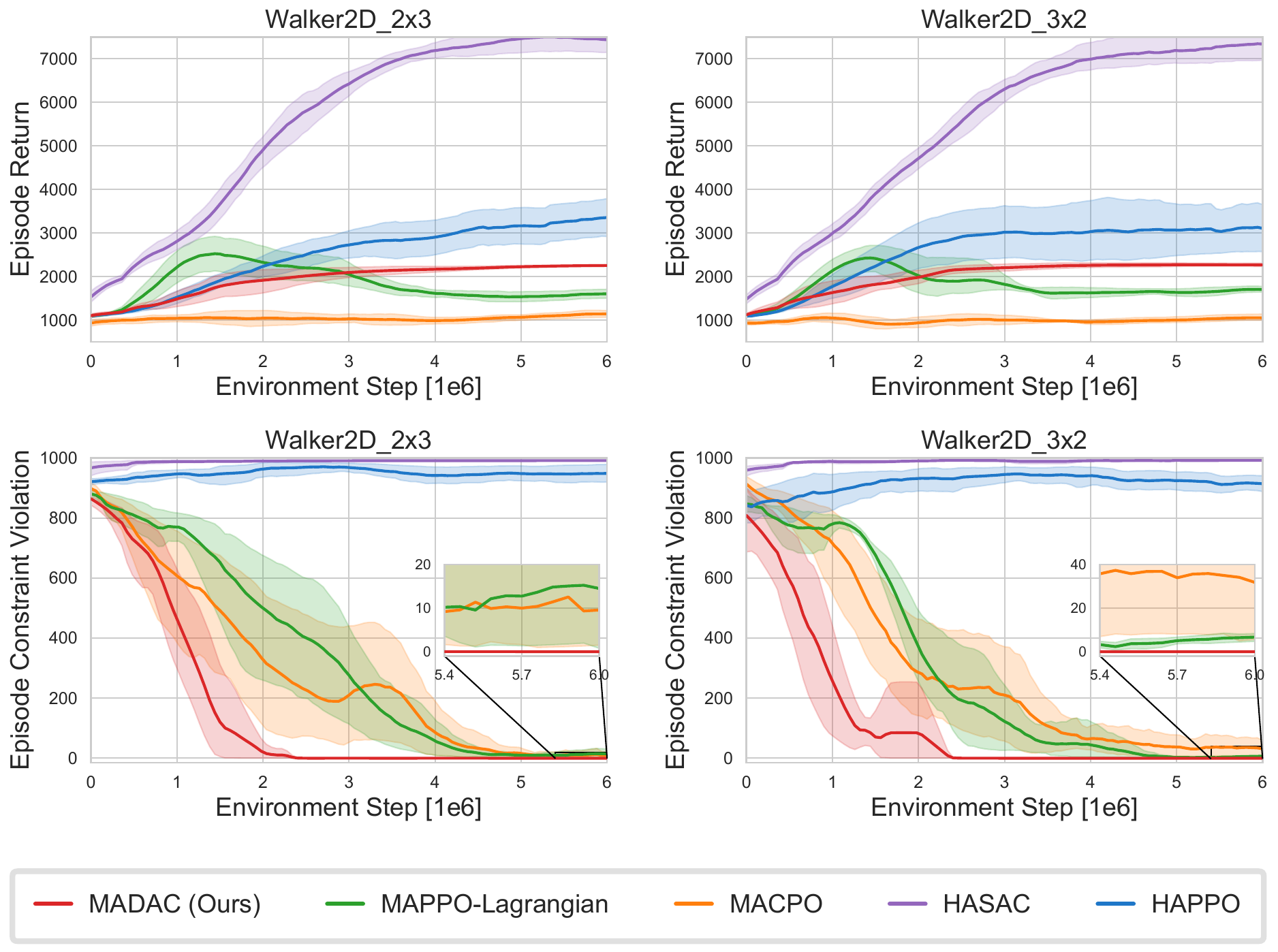}
    \caption{Training curves on Walker2D environments. While standard MARL baselines (HASAC and HAPPO) achieve higher rewards, they violate the constraints heavily. MADAC outperforms safe MARL baselines (MACPO and MAPPO-Lagrangian) by achieving significantly higher rewards, while maintaining equal or superior compliance with safety constraints.}
    \label{Walker2D_Results}
\end{figure*}

\begin{figure*}[htbp]
    \centering
    \includegraphics[height=4.25in]{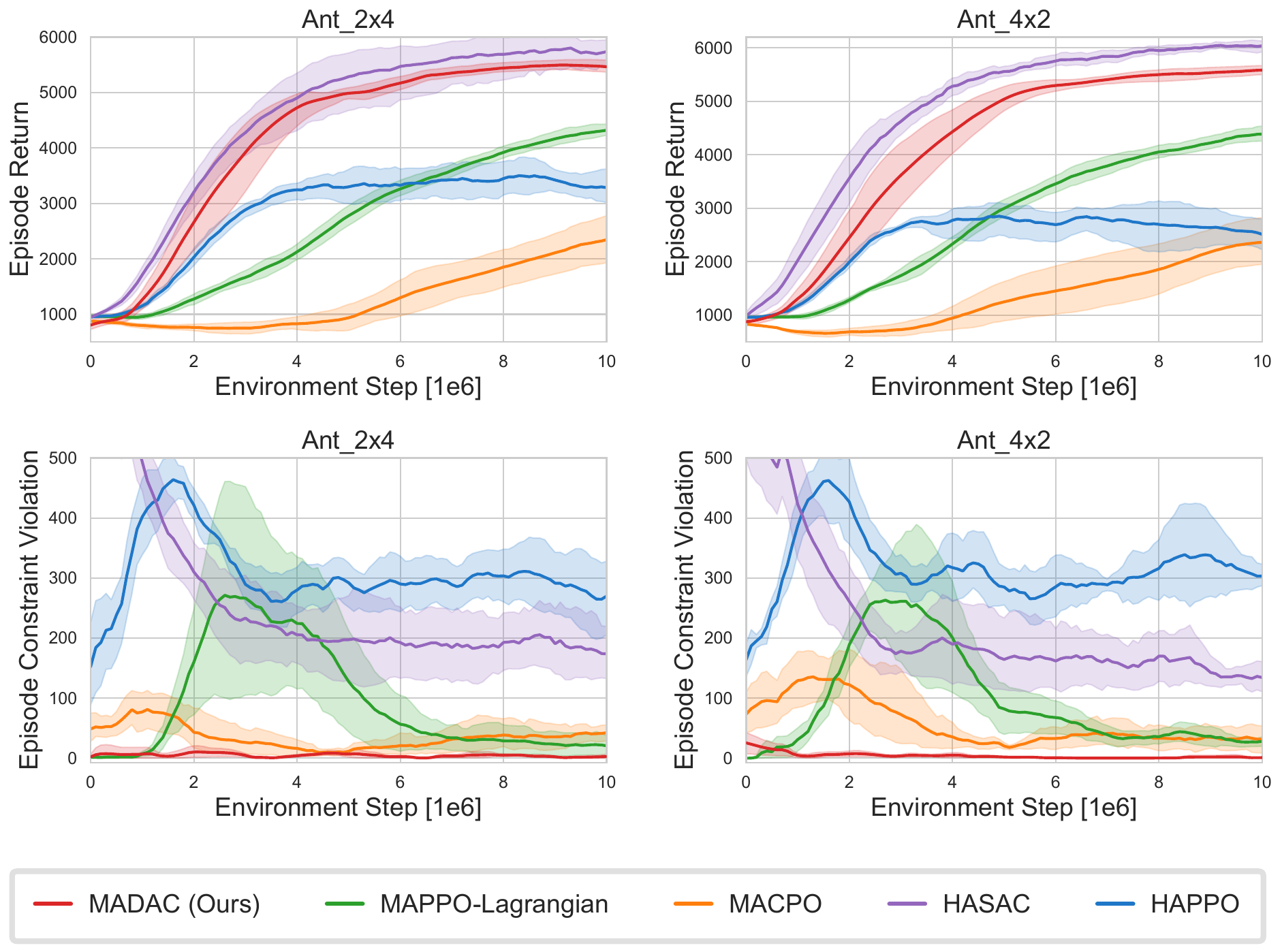}
    \caption{Training curves on Ant environments. Standard MARL baselines (HASAC and HAPPO) violate the constraints heavily. MADAC achieves rewards comparable to HASAC and substantially higher than HAPPO, while adhering to safety requirements. MADAC outperforms safe MARL baselines (MACPO and MAPPO-Lagrangian) by achieving significantly higher rewards, while maintaining equal or superior compliance with safety constraints.}
    \label{Ant_Results}
\end{figure*}

We conduct agent-by-agent sequential updates for task policies and safety policies following the methodology in Algorithm~\ref{Multi-agent dual policy iteration}. Given a specific agent order $i_{1:n}$, the loss function of the $i_k$-th agent's safety policy is
\begin{equation*}
    \begin{aligned}
        &L_{\pi_{h,i_k}}\left(\phi_{i_k}\right) = \\&\mathbb{E}_{x\sim\mathcal{D}}\left\{-\min_{j\in\left\{1,2\right\}}\left\{Q_h\left(x,\left(\bar{u}_{i_{1:k-1}}^{\rm{new}},u_{i_k},\bar{u}_{i_{k+1:n}}\right);\psi_j\right)\right\}\right\},
    \end{aligned}
\end{equation*}
where $\bar{u}_{i_{1:k-1}}^{\rm{new}}$ is given by
\begin{equation*}
    \bar{u}_{i_{1:k-1}}^{\rm{new}}=\left(\pi_{h,i_1}(x;\phi^{\rm{new}}_{i_1}),\cdots,\pi_{h,i_{k-1}}(x;\phi^{\rm{new}}_{i_{k-1}})\right),
\end{equation*}
$u_{i_k}= \pi_{h,i_k}(x;\phi_{i_k})$, and $\bar{u}_{i_{k+1:n}}$ is given by
\begin{equation*}
    \bar{u}_{i_{k+1:n}}=\left(\pi_{h,i_{k+1}}(x;\phi_{i_{k+1}}),\cdots,\pi_{h,i_n}(x;\phi_{i_n})\right).
\end{equation*}
The notation $\bar{u}$ indicates that there is no gradient back-propagation through $u$. The notation $\phi^{\rm{new}}_{i_m}$ denotes that we are using the updated parameters, i.e.,
\begin{equation*}
    \phi^{\rm{new}}_{i_m}= \phi_{i_m} - \beta \nabla_{\phi_{i_m}} L_{\pi_{h,i_m}}\left(\phi_{i_m}\right),
\end{equation*}
in which $\beta$ is the learning rate. Minimizing $L_{\pi_{h,i_k}}$ maximizes the estimated safety value of the successor state, and hence the objective of the safety policy improvement in Algorithm~\ref{Multi-agent safety policy iteration}.

The updates of the task policies depend on whether a state sample $x$ belongs to the CIS, which is tested by $h(x)\geq 0$ and $\min_{j\in\left\{1,2\right\}}\left\{Q_h(x,\pi_{h}(x;\phi);\psi_j)\right\} \geq 0$, corresponding to $V_h^{\pi_h}(x)\geq 0$ by (\ref{Self-consistency Condition for safety value function}). Let $\mathcal{D}_{\rm{out}}$ denote the samples that fail this test before the safety value networks and safety policies are updated, and let $\mathcal{D}_{\rm{in}}^{\rm{new}}$ and $\mathcal{D}_{\rm{out}}^{\rm{new}}$ denote the samples that pass and fail it afterwards. These groups correspond to $\mathcal{X}\setminus S_{\rm{c}}$, $S_{\rm{c}}^{\rm{new}}$, and $\mathcal{X}\setminus S_{\rm{c}}^{\rm{new}}$ in Algorithm~\ref{Multi-agent dual policy iteration}, and expectations over them denote averages over the respective groups. If a group is empty, the corresponding loss is skipped.

For the feasibility enforcement, the loss function of the $i$-th agent's task policy is
\begin{equation*}
    L_{\pi_{i}}^{\rm{fe}}\left(\theta_{i}\right)=\mathbb{E}_{x\sim\mathcal{D}_{\rm{out}}}\left\{\left\|u_{i}-\bar{u}_{i}\right\|_2^2\right\},
\end{equation*}
where $u_{i}\sim \pi_{i}(x;\theta_{i})$ and $\bar{u}_{i}=\pi_{h,i}(x;\phi_{i})$. This loss function design encourages the task policy to mimic the safety policy at states outside the current CIS.

For the task policy improvement, the loss function of the $i_k$-th agent's task policy contains two components:
\begin{equation*}
    L_{\pi_{i_k}}\left(\theta_{i_k}\right)=L_{\pi_{i_k}}^{\rm{in}}\left(\theta_{i_k}\right)+L_{\pi_{i_k}}^{\rm{out}}\left(\theta_{i_k}\right).
\end{equation*}
$L_{\pi_{i_k}}^{\rm{in}}\left(\theta_{i_k}\right)$ and $L_{\pi_{i_k}}^{\rm{out}}\left(\theta_{i_k}\right)$ correspond to the $x\in S_{\rm{c}}^{\rm{new}}$ and $x\notin S_{\rm{c}}^{\rm{new}}$ cases in Algorithm~\ref{Multi-agent dual policy iteration}, respectively. Similar to the feasibility enforcement, $L_{\pi_{i_k}}^{\rm{out}}\left(\theta_{i_k}\right)$ is given by
\begin{equation*}
    L_{\pi_{i_k}}^{\rm{out}}\left(\theta_{i_k}\right)=\mathbb{E}_{x\sim\mathcal{D}_{\rm{out}}^{\rm{new}}}\left\{\left\|u_{i_k}-\bar{u}_{i_k}\right\|_2^2\right\},
\end{equation*}
where $u_{i_k}\sim \pi_{i_k}(x;\theta_{i_k})$ and $\bar{u}_{i_k}=\pi_{h,i_k}(x;\phi_{i_k})$. For states inside the latest CIS, we need to conduct constrained policy optimization.
$L_{\pi_{i_k}}^{\rm{in}}\left(\theta_{i_k}\right)$ is given by
\begin{equation*}
    \begin{aligned}
        &L_{\pi_{i_k}}^{\rm{in}}\left(\theta_{i_k}\right)\\=&\mathbb{E}_{x\sim\mathcal{D}_{\rm{in}}^{\rm{new}}}\left\{\alpha \log \pi_{i_k}(u_{i_k} \mid x;\theta_{i_k})- \min_{j\in\left\{1,2\right\}}\left\{Q(x, u;\omega_j)\right\}\right\} \\&+ \mathbb{E}_{x\sim\mathcal{D}_{\rm{in}}^{\rm{new}}}\left\{-\lambda_{i_k} \min_{j\in\left\{1,2\right\}}\left\{Q_h(x, u;\psi_j)\right\}\right\},
    \end{aligned}
\end{equation*}
in which $u=\left(\bar{u}_{i_{1:k-1}}^{\rm{new}},u_{i_k},\bar{u}_{i_{k+1:n}}\right)$,
where $\bar{u}_{i_{1:k-1}}^{\rm{new}}$ is given by
\begin{equation*}
    \begin{aligned}
        &\bar{u}_{i_{1:k-1}}^{\rm{new}}=\left(\bar{u}_{i_{1}}^{\rm{new}},\cdots,\bar{u}_{i_{k-1}}^{\rm{new}}\right),\\&
        \bar{u}_{i_{m}}^{\rm{new}}\sim \pi_{i_m}(x;\theta_{i_m}^{\rm{new}}),\quad 1\leq m\leq k-1,
    \end{aligned}
\end{equation*}
$u_{i_k}\sim \pi_{i_k}(x;\theta_{i_k})$, and $\bar{u}_{i_{k+1:n}}$ is given by
\begin{equation*}
    \begin{aligned}
        &\bar{u}_{i_{k+1:n}}=\left(\bar{u}_{i_{k+1}},\cdots,\bar{u}_{i_n}\right),\\&
        \bar{u}_{i_{m}}\sim \pi_{i_m}(x;\theta_{i_m}),\quad k+1\leq m\leq n.
    \end{aligned}
\end{equation*}
Similar to the case of safety policies, the notation $\theta^{\rm{new}}_{i_m}$ denotes that we are using the updated parameters, i.e.,
\begin{equation*}
    \theta^{\rm{new}}_{i_m}= \theta_{i_m} - \beta \nabla_{\theta_{i_m}} L_{\pi_{i_m}}\left(\theta_{i_m}\right).
\end{equation*}
And the notation $\bar{u}$ indicates that there is no gradient back-propagation through $u$.
The loss function for the Lagrange multiplier $\lambda_{i_k}$ is given by
\begin{equation*}
    L_{\lambda_{i_k}} =
    \mathbb{E}_{x\sim\mathcal{D}_{\rm{in}}^{\rm{new}}}\left\{\lambda_{i_k} \min_{j\in\left\{1,2\right\}}\left\{Q_h(x, u;\psi_j)\right\}\right\},
\end{equation*}
where the notations are consistent with those in the task policy loss function. This multiplier loss and the task policy loss $L_{\pi_{i_k}}^{\rm{in}}\left(\theta_{i_k}\right)$ constitute a dual ascent approach, which mimics the constrained policy updates in Algorithm~\ref{Multi-agent dual policy iteration}.

\begin{remark}
The theoretical results in the previous section are established for finite state and action spaces under exact policy evaluation and exact sequential policy updates. In the deep RL implementation presented here, these tabular objects are replaced by function approximators and stochastic optimization on the collected samples. As a result, Algorithm~\ref{MADAC} does not, in general, inherit the \emph{strict} convergence guarantees enjoyed by Algorithms~\ref{Multi-agent safety policy iteration} and~\ref{Multi-agent dual policy iteration}.
Nevertheless, the theory remains essential for algorithm design. (i) It provides a principled derivation of the loss functions and update rules in Algorithm~\ref{MADAC}, directly from the underlying self-consistency equations and constrained sequential updates; without these results, ad hoc heuristic designs often lead to unstable training and poor outcomes. (ii) It clarifies how safety and performance objectives should be separated across feasible and infeasible regions, and how to coordinate agent updates with feasibility-preserving mechanisms, thereby guiding implementation and providing diagnostics for training on high-dimensional systems.
\end{remark}

\section{Experiments}

In this section, we evaluate the proposed MADAC algorithm on safety-critical MARL benchmarks, comparing it with state-of-the-art MARL algorithms in terms of reward maximization and safety preservation.

The safety-critical benchmarks used in our experiments are based on MuJoCo \cite{todorov2012mujoco}, following the setups from \cite{gu2023safe} and \cite{ji2023safety}. We utilize three robots—HalfCheetah, Walker2D, and Ant—each with action space partitioned into two multi-agent configurations, resulting in a total of six different environments. Detailed information is provided below.

\textbf{HalfCheetah} is a two-dimensional robot with nine links and eight joints connecting them (including two paws), as shown in Fig.~\ref{HalfCheetah pic}.
The observation space consists of $17$ dimensions.
The goal is to make the robot move forward as fast as possible. The reward function includes a term proportional to the change in $x$ coordinate and a small penalty for control energy consumption. State constraints are defined in two areas: the robot must avoid tipping over, with the torso angle $\theta$ constrained to $-0.3\leq \theta \leq 0.3$, and it must not exceed a specified maximum speed, with the torso velocity $v$ limited to $v \leq 2.5$.
The joint action space has $6$ dimensions, representing the torques applied to six joints. We configure this robot in two ways: \textbf{HalfCheetah-2x3}, where two agents each control three joints, and \textbf{HalfCheetah-3x2}, where three agents each control two joints.

\textbf{Walker2D} is a two-dimensional robot with four main body parts (one torso, two thighs, and two legs), as shown in Fig.~\ref{Walker2D pic}.
The observation space is also $17$-dimensional. The task is to make the robot run forward as fast as possible. The reward function includes a term proportional to the change in $x$ coordinate, along with a small penalty for energy used in control. The state constraints ensure the robot remains upright, with the torso height $z$ constrained to $1.0\leq z \leq 1.8$, and that it does not exceed a maximum speed, limiting the torso velocity $v$ to $v \leq 1.5$.
The joint action space has $6$ dimensions corresponding to the torque inputs at the hinge joints. We define two configurations for this robot: \textbf{Walker2D-2x3}, where two agents control three joints each, and \textbf{Walker2D-3x2}, where three agents control two joints each.

\textbf{Ant} is a three-dimensional robot with a central torso and four legs, each consisting of two links, as shown in Fig.~\ref{Ant pic}.
The observation space is $29$-dimensional. The objective is to make the robot move forward along the $x$ axis as quickly as possible. The reward function includes a term proportional to the change in the $x$ coordinate, a small penalty for energy consumption, and a minor penalty for the contact force strength with the ground.
The state constraints ensure that the robot does not fall or flip over, with the torso height $z$ constrained to $0.2\leq z \leq 1.0$, and the torso's tilt angle $\phi$ limited to $\cos\phi \geq -0.7$. Additionally, the robot must avoid collisions with zigzag walls, as designed in \cite{gu2023safe}.
The joint action space is $8$-dimensional, representing the torques applied at the hinge joints. We define two configurations for this robot: \textbf{Ant-2x4} and \textbf{Ant-4x2}. In the former, two agents control four joints each, and in the latter, four agents control two joints each.

We compare MADAC with two standard MARL algorithms as well as two safe MARL algorithms.
\textbf{HAPPO} \cite{zhong2024heterogeneous} is a state-of-the-art on-policy MARL method that extends the standard PPO algorithm \cite{schulman2017proximal} to multi-agent settings with an agent-by-agent sequential update scheme, facilitating convergence to a Nash equilibrium. \textbf{HASAC} \cite{liu2024maximum} is a state-of-the-art off-policy MARL method with a similar structure to HAPPO but based on the maximum entropy framework of the SAC algorithm \cite{haarnoja2018soft}.
\textbf{MACPO} \cite{gu2023safe} is a state-of-the-art safe MARL method, which is the multi-agent extension of the standard CPO algorithm \cite{achiam2017constrained}, using a trust-region approach for constrained policy updates. \textbf{MAPPO-Lagrangian} \cite{gu2023safe} is a state-of-the-art safe MARL method that extends the PPO-Lagrangian algorithm \cite{ray2019benchmarking} to multi-agent scenarios, applying Lagrange multipliers for enforcing policy constraints.

The hyperparameters for HAPPO and HASAC are taken from the tuned configurations for standard (unconstrained) multi-agent MuJoCo environments in the open-source implementation provided by \cite{zhong2024heterogeneous, liu2024maximum}. The hyperparameters for MACPO and MAPPO-Lagrangian are taken from the tuned configurations in the open-source implementation provided by \cite{gu2023safe}. For each algorithm, we use the same set of hyperparameters across all environments, except for one adjustment with the Walker2D robot, where the learning rate of the Lagrange multipliers is increased for both MAPPO-Lagrangian and MADAC. Additionally, we do not terminate an episode when a constraint is violated, as this would blur the distinction between the algorithm's reward-seeking capability and its safety-preserving capability. We expect the algorithm to discover safe behaviors on its own, which aligns with the setting in \cite{ji2023safety, ray2019benchmarking}. An exception is made for MACPO in the Ant environments; without it, MACPO fails to learn, resulting in extremely negative episode returns.

We adopt two evaluation metrics in the training process: episode return and episode constraint violation. A higher return indicates better task performance, and a lower constraint violation indicates better safety.
The training results of all algorithms for the six environments are shown in Figs.~\ref{HalfCheetah_Results}, \ref{Walker2D_Results}, and~\ref{Ant_Results}, which plot both metrics against the number of collected samples during training. The solid lines correspond to the mean and the shaded regions correspond to 95\% confidence interval over five seeds. The training curves reflect the reward-seeking and safety-preserving capabilities of the learned policies as training progresses.
The standard MARL algorithms, HAPPO and HASAC, violate constraints heavily in all environments, highlighting the need for properly designed safe MARL algorithms for safety-critical tasks.
In general, among the safe MARL algorithms, our algorithm (MADAC) significantly outperforms MACPO and MAPPO-Lagrangian, achieving much higher rewards and fewer constraint violations.
In the HalfCheetah environments, MAPPO-Lagrangian achieves similar safety-preserving performance to MADAC. However, in the more complex Walker2D and Ant environments, MAPPO-Lagrangian violates constraints more frequently than MADAC. MACPO performs the poorest among the three algorithms, as it struggles to learn a zero-violation policy even in the simplest HalfCheetah environments.
A notable observation is that, in the HalfCheetah environments, MACPO shows fewer constraint violations during the early stages of training. This is because MADAC and MAPPO-Lagrangian, with their stronger reward-seeking capabilities, quickly learn to drive the robot faster (which violates the safety constraints heavily). As training progresses, both algorithms start to constrain the robot’s speed to meet the safety threshold.
Another important observation is MADAC's significant ability to achieve high rewards while maintaining safety. In the Ant environments, MADAC reaches returns comparable to HASAC but with significantly lower constraint violations. This demonstrates the effectiveness of MADAC's underlying mechanism, multi-agent dual policy iteration, which seeks generalized Nash equilibria to maximize rewards while ensuring joint constraint satisfaction.

\section{Conclusion}

In this paper, we presented a comprehensive framework for solving state-wise safe MARL problems, covering both theoretical analysis and algorithm design. We introduced a multi-agent approach for identifying controlled invariant sets that monotonically converges to a Nash equilibrium on safety value functions. We then embedded this method within a multi-agent dual policy iteration scheme that guarantees convergence to a generalized Nash equilibrium for state-wise-constrained cooperative Markov games. Furthermore, we proposed the multi-agent dual actor-critic algorithm for effective deployment in complex high-dimensional systems. Experimental results on safety-critical MARL benchmarks demonstrate the effectiveness of our method.

\section*{Acknowledgment}

We sincerely thank Stephanie Gil of Harvard University for insightful discussions during the early stages of this work, and Juan Cervino for his valuable feedback on the manuscript.

\bibliographystyle{IEEEtran}
\bibliography{reference}

\vfill

\end{document}